\documentclass[11pt]{article}

\PassOptionsToPackage{table}{xcolor} 
\usepackage[final]{acl}

\usepackage{times}
\usepackage{latexsym}

\usepackage[T1]{fontenc}

\usepackage[utf8]{inputenc}

\usepackage{microtype}

\usepackage{inconsolata}

\usepackage{graphicx}

\usepackage{booktabs}
\usepackage{amsfonts}
\usepackage{nicefrac}
\usepackage{amsmath}
\usepackage{subfigure}
\usepackage{amsthm}
\usepackage{multirow}
\usepackage{tcolorbox}
\usepackage{makecell}
\usepackage{pifont}

\newtheorem{assumption}{Assumption}

\newtheorem{proposition}{Proposition}
\newtheorem{corollary}{Corollary}
\newtheorem{remark}{Remark}
\newtheorem*{proposition*}{Proposition}

\definecolor{lightroyalblue}{HTML}{F6F8FD}
\definecolor{royalblue}{HTML}{4169E1}
\definecolor{lighterblue}{HTML}{f2fafd}
\definecolor{LightCyan}{rgb}{.9, .95, 1.}
\definecolor{mydarkred}{rgb}{0.6,0,0}
\definecolor{mydarkgreen}{rgb}{0,0.6,0}
\newtcolorbox{abox}{colback=lightroyalblue,colframe=black}
\DeclareMathOperator{\Var}{Var}

\def\eqref#1{equation~\ref{#1}}
\def\Eqref#1{Equation~\ref{#1}}

\title{Uncertainty-Aware Budget Allocation for Efficient LLM Tree Search}

\author{
  \textbf{Jianguo Huang},
  \textbf{Zhiyi Lyu},
  \textbf{Fuxiang Zhang}, 
  \textbf{Yanchen Deng}\thanks{Corresponding author.},
  \textbf{Deheng Ye},
  \textbf{Bo An}\\
  Nanyang Technological University, Singapore
}

\begin{document}

\maketitle

\begin{abstract}
Tree search enhances multi-step reasoning in large language models (LLMs), yet prevailing methods expand every step with a fixed budget, irrespective of the marginal utility of expansion. We first show, across existing search algorithms, that continuations sampled from a shared intermediate state are largely duplicates and collapse to a few distinct answers. Critically, we find that the semantic dispersion of sibling actions predicts the redundancy of their downstream continuations, permitting a subtree's budget to be fixed before expansion. An information-theoretic analysis motivates this: treating the per-step budget as the Monte Carlo sample size for estimating a step's residual answer entropy, we derive a sufficient budget that shrinks with sibling dispersion. We accordingly propose \textbf{Uncertainty-Aware Allocation} (UAA), a training-free framework that decouples value from uncertainty: values determine \emph{which} candidates are promising, while uncertainty, estimated from sibling dispersion at no additional model cost, determines \emph{how many} merit expansion, via uncertainty-aware pruning and uncertainty-aware budgeting. Across mathematical reasoning, logical reasoning, and planning benchmarks, UAA reduces token usage by up to $90\%$ and latency by up to $9\times$ at comparable accuracy, generalizing across four tree search algorithms and both open- and closed-source LLMs.
\end{abstract}

\section{Introduction}
\label{sec:intro}
Large language models (LLMs) have demonstrated impressive capabilities in natural language reasoning~\citep{deepseek-ai2025deepseekr1, yang2025qwen3, nvidia2025nvidia,singh2025openai}, but they still struggle with tasks that demand sustained and accurate multi-step reasoning~\citep{ ahn2024large, wei2025satbench, lin2025zebralogic, banerjee2025synthtrips}.
To address this, Chain-of-Thought (CoT)~\citep{wei2022chain} guides LLMs to decompose the problem-solving process into multiple intermediate reasoning steps. CoT improves interpretability and performance, but commits to a single linear chain with no mechanism to backtrack, so that a single flawed step can cascade and corrupt the final answer.

This brittleness motivates tree-based search methods, which maintain multiple candidate trajectories to enable more robust reasoning. Specifically, tree search algorithms, including Beam Search~\citep{yao2023tree,xie2023self,kang2024mindstar,snell2024scaling}, and Monte Carlo Tree Search (MCTS)~\citep{hao2023reasoning, feng2023alphazero, jiang2024technical, wang2024q}, have emerged as a promising test-time scaling (TTS) paradigm to find high-value reasoning trajectories by explicitly exploring the space of the intermediate reasoning steps. However, given the vast space of reasoning steps, these methods often incur significant computational overhead~\citep{grosse2024uncertainty,wang2025don,lyu2025let}, making them impractical for real-world deployment. For example, MCTS typically requires many iterations to balance exploration and exploitation, inflating the token and latency cost.
Besides the intractable search space, another key factor is the flexibility of budget allocation. Most existing tree search algorithms adopt a uniform budgeting scheme, i.e., expanding a fixed number of children at each decision step without considering the actual utility (e.g., difficulty, uncertainty, etc.). This rigid strategy is inherently wasteful: trivial steps consume the same expansion budget as genuinely challenging ones, leading to inefficient use of test-time compute. Our pilot study in this work reveals that the majority of continuations sampled from the same intermediate state have pairwise cosine similarity exceeding 0.95, confirming widespread redundancy in the reasoning tree, highlighting the need for strategic allocation of budget over decision steps. 

Several attempts have been made to introduce more flexible budget allocation during tree search. Recent approaches~\citep{wang2024litesearch, wang2025don, liu2025equivpruner} merge similar reasoning steps or prune redundant ones during search. By leveraging node-level importance scores, these methods reduce unnecessary computation and lower overall token consumption.
However, the existing acceleration techniques still face significant challenges in computational cost and generality. Specifically, they often require training additional models to estimate step-level similarities, which not only brings extra training and serving overheads but also lacks transferability across different reasoning tasks and datasets. Moreover, most existing approaches merge or prune redundant actions \emph{locally} at the current reasoning step~\citep{tian2024toward,wang2025don}, without exploiting a more informative structure: the semantic redundancy among sibling actions at the current step predicts the redundancy of their downstream continuations. As a result, these methods tend to over-explore semantically saturated regions of the search space, wasting test-time compute. These limitations raise a natural question: \textit{Can we accelerate tree-based reasoning by pruning the predictable redundancy of downstream trajectories from internal signals alone, without external models and without compromising performance?}

To support this intuition, we provide an information-theoretic analysis showing that, under a Lipschitz assumption, a smaller sampling budget suffices to estimate a step's residual answer-entropy to a fixed precision with high probability when the step uncertainty is low. This result motivates a non-uniform allocation policy: rather than expanding every step with the same budget, compute should be spared at low-uncertainty steps and preserved for high-uncertainty ones, where additional exploration meaningfully improves answer coverage. Motivated by these insights, we propose \textbf{Uncertainty-Aware Allocation} (UAA), a training-free, plug-and-play framework that dynamically allocates the expansion budget at each reasoning step based on its uncertainty. UAA uses cosine similarity among sibling actions as a lightweight proxy for uncertainty; crucially, this current-step dispersion is predictive of the redundancy in the downstream continuations (Appendix~\ref{appendix:current_step}), letting UAA set the budget for a subtree before expanding it. It reduces unnecessary exploration via two complementary mechanisms: (i) uncertainty-aware pruning (UAP), which retains only a high-quality subset of candidate actions at low-uncertainty steps, and (ii) uncertainty-aware budgeting (UAB), which shrinks the next-step expansion budget for retained low-uncertainty actions.

\begin{enumerate}
    \item We find that continuations from the same intermediate state in tree search are often semantically redundant, and show this holds beyond math (StrategyQA, planning) and beyond a single algorithm/model (MCTS, Qwen3-32B).
    \item We propose \emph{Uncertainty-Aware Allocation} (UAA), a \emph{training-free}, \emph{plug-and-play} framework that separates value from uncertainty: values decide \emph{which} candidates are promising, while uncertainty decides \emph{how many} candidates are worth expanding.
    \item Across Beam Search, MCTS, DVTS, and REBASE, UAA consistently cuts token usage and latency while keeping accuracy comparable, generalizing across Llama3.1-8B-Instruct, Qwen3-32B, and GPT-4o under both white-box and black-box uncertainty proxies.
\end{enumerate}

\begin{figure*}[t]
    \centering
    \subfigure[Similarity of reasoning trajectories]{
        \label{fig:path_similarity}
        \includegraphics[width=0.31\linewidth]{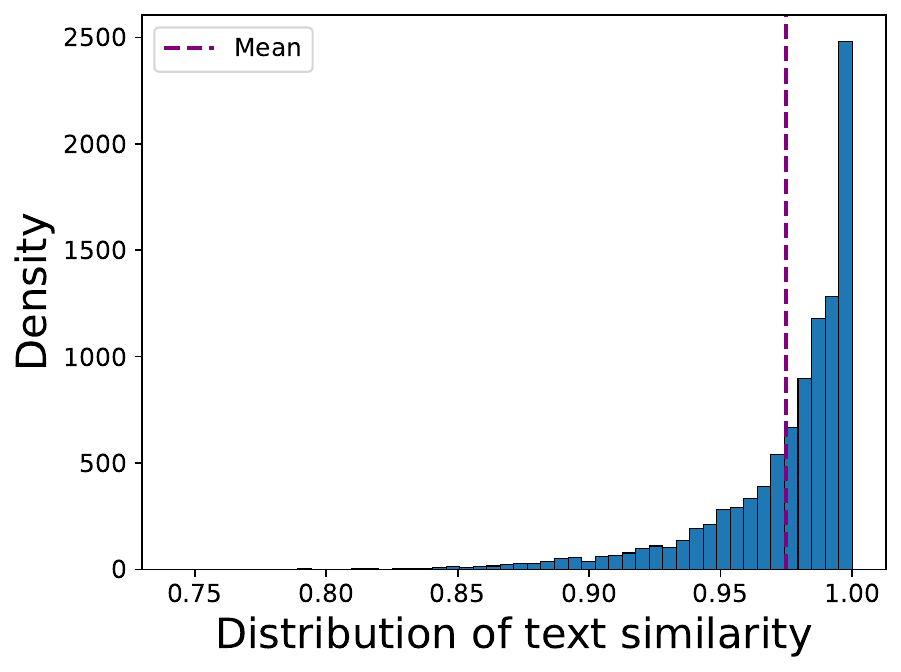}
    }
    \subfigure[Spearman correlation between continuation similarity and answer diversity]{
        \label{fig:path_similarity_diversity_correlation}
        \includegraphics[width=0.31\linewidth]{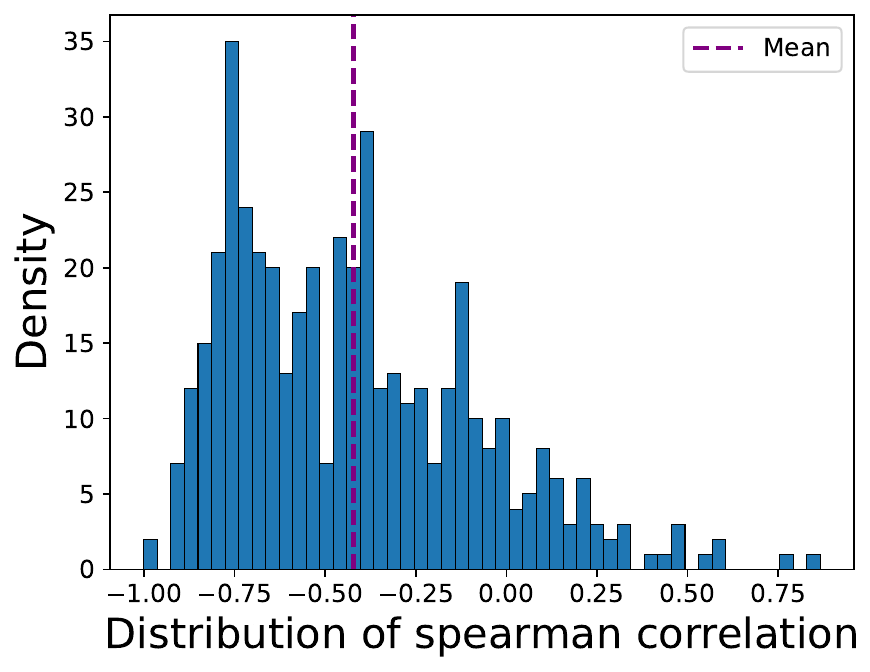}
    }
    \subfigure[Spearman correlation between semantic dispersion of sibling actions and answer accuracy]{
        \label{fig:path_similarity_acc_correlation}
        \includegraphics[width=0.31\linewidth]{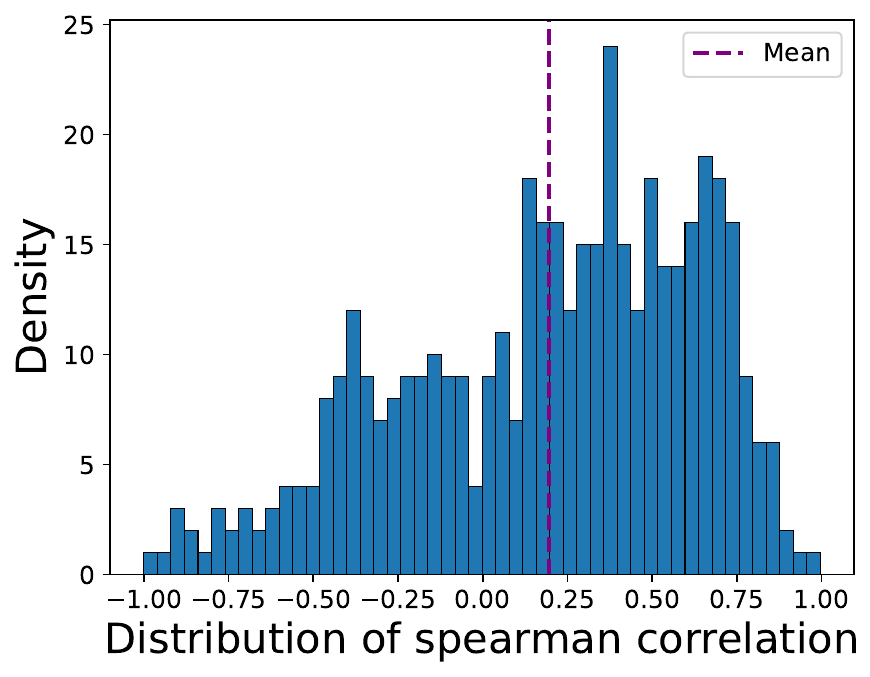}
    }
    \caption{Motivational analysis of reasoning-tree redundancy on \textbf{GSM8K} using reward-guided Beam Search with \textit{LLaMA3.1-8B-Instruct} (beam width $20$, $20$ children per node); continuations are embedded with the \mbox{gte-Qwen2-7B-instruct} encoder. \textbf{(a)} Pairwise cosine similarity between continuations from the same parent state. \textbf{(b)} Per-parent Spearman correlation between continuation similarity and the number of unique final answers. \textbf{(c)} Per-parent Spearman correlation between the semantic dispersion of sibling actions and downstream answer accuracy.}
    \label{fig:reasoning_tree_motivation}
\end{figure*}

\section{Preliminaries}
\label{sec:preliminaries}

\subsection{Reasoning as a sequential decision process}

Given an input question $q$, the reasoning process is decomposed into a sequence of $T$ steps $\mathbf{a} := [a_1, \dots, a_T]$, where each $a_t$ is a line of text or a fixed number of tokens. We model this as a Markov process $\langle \mathcal{S}, \mathcal{A}, \mathcal{T}, \pi_\theta \rangle$, where the state $s_t := [q, a_1, \dots, a_{t-1}]$ is the partial reasoning trace, and the transition is deterministic: $s_{t+1} = [s_t, a_t]$. The LLM policy $\pi_\theta$ defines the distribution $\pi_\theta(a_t \mid s_t)$ over the next step, and a complete reasoning trajectory $[a_1, \dots, a_T]$ terminates with a final answer $y \in \mathcal{Y}$.

\subsection{Tree search for LLM reasoning}
\label{sec:tree_search}

Tree search algorithms~\citep{lightman2023let, wang2023math, yao2023tree, hao2023reasoning, wang2024q} navigate the reasoning space guided by a value function $V$, seeking the optimal trajectory:
\begin{equation*}
    \mathbf{a}^* = \mathop{\arg\max}_{[a_1,\dots,a_T]\in\mathcal{A}^T} V(q, a_1, \dots, a_T).
\end{equation*}
Each node in the search tree corresponds to a state $s_t$, each edge to an action $a_t$, and every root-to-leaf path constitutes a complete trajectory. Beam Search~\citep{xie2023self, feng2023alphazero} maintains $N$ candidates per step: it retains the $N/B$ candidates with the highest value and samples $B$ continuations from each, which restores $N$ candidates for the next step.

\subsection{Value functions}
\label{sec:value_functions}
We adopt process-level value functions, which score intermediate reasoning steps rather than only the final answer. Specifically, we consider two such functions~\citep{yang2018breaking, hao2023reasoning, snell2024scaling}:

\textbf{(i) Model Confidence} measures the LLM's belief in a trajectory via the average log-probability of generated tokens~\citep{yang2018breaking, hao2023reasoning, snell2024scaling, huang2023look, fu2025deep}:
\begin{equation}\label{eq:confidence}
    \resizebox{0.87\columnwidth}{!}{$\displaystyle
    V_c(s_t, a_t) = \frac{\sum_{z=1}^{t} \sum_{k=1}^{\ell_{a_z}} \log \pi_\theta(b_{z,k} \mid s_z, b_{z,1}, \dots, b_{z,k-1})}{\sum_{z=1}^{t} \ell_{a_z}}.
    $}
\end{equation}

\textbf{(ii) Process Reward Models (PRMs)}~\citep{lightman2023let, wang2023math} assign a scalar score $V_r(s_t, a_t)$ to each reasoning step, offering more targeted supervision than outcome-only feedback. PRMs are more robust to incorrect steps, but require a separately trained model.

\section{Motivation}
\label{sec:motivation}
Tree-search algorithms for LLM reasoning typically expand each intermediate state with a fixed number of candidates, regardless of how much new information those candidates carry. In fact, these candidates are often near-duplicates, so a uniform budget wastes compute on steps that add little new reasoning. We make this precise through an empirical pilot study and an information-theoretic analysis linking redundancy to the budget a step actually needs.

\subsection{Empirical pilot study}
\label{sec:pilot_study}

We conduct an empirical study on GSM8K~\citep{cobbe2021training} to reveal the ubiquitous redundancies among reasoning trajectories in tree-search methods, and demonstrate that such redundancies reduce both answer diversity and downstream accuracy. Specifically, given a partial reasoning trace $s_t=[q,a_1,\dots,a_{t-1}]$, we define its continuation as $p^c_{s_t}=[a_t,\dots,a_T]$. The redundancy is measured by first independently sampling $m$ continuations $p^c_{s_t,1},\dots,p^c_{s_t,m}$, and then computing pairwise cosine similarity between their text embeddings extracted from a pre-trained encoder. The analysis is performed over all intermediate states generated during the search process.

Figure~\ref{fig:reasoning_tree_motivation} presents the results of reward-guided beam search~\citep{yao2023tree,snell2024scaling} on \textit{LLaMA3.1-8B-Instruct}~\citep{grattafiori2024llama} (beam width 20, 20 children per node; details in Appendix~\ref{appendix:motivation}). The majority of continuations have a similarity score above 0.95, and the mean similarity approaches 1, confirming prevalent redundancy in the reasoning trajectories.

We further explore how this redundancy correlates with answer quality. First, redundant continuations produce fewer distinct final answers: the Spearman correlation between pairwise similarity and the number of unique answers is predominantly negative with a mean close to $-0.5$ (Figure~\ref{fig:path_similarity_diversity_correlation}). Second, steps with higher semantic dispersion among sibling actions yield higher downstream accuracy, whereas highly redundant steps show negligible marginal accuracy gains under the same budget (Figure~\ref{fig:path_similarity_acc_correlation}). Together, these results establish that \textbf{exploring a highly certain reasoning step whose continuations are semantically similar contributes little to either answer diversity or accuracy}, highlighting the need for an uncertainty-aware allocation policy. More empirical studies on MCTS, Qwen3-32B~\citep{yang2025qwen3}, and StrategyQA~\citep{geva2021did} can be found in Appendix~\ref{appendix:more_empirical_analysis}.

\subsection{Information-theoretic analysis}
\label{sec:theoretical_analysis}

We now formalize the above observation: the number of samples sufficient to estimate a step's residual answer-entropy shrinks as its uncertainty decreases, motivating a non-uniform expansion policy. Let $Y$ be the final answer that the policy $\pi_\theta$ reaches when it continues reasoning from a state; $Y$ is random through the sampling of the remaining steps, with $q$ and $s_t$ fixed. It is the model's \emph{predicted} answer rather than the ground-truth label, which is constant given $q$ and would make every conditional entropy zero. We measure the utility of expanding $A_t$ by the residual answer-entropy $H(Y\mid A_t, S_t=s_t)$, the model's remaining uncertainty about its own answer after committing to $A_t$: for a fixed state $s_t$, a smaller residual entropy means the expansion resolves more of the answer uncertainty~\citep{ton2024understanding}. We estimate this quantity by Monte Carlo: we sample $n$ candidate actions $\{a_t^{(i)}\}_{i=1}^n \stackrel{\text{i.i.d.}}{\sim} \pi_\theta(\cdot\mid s_t)$ and form $\widehat{H}_n := \tfrac{1}{n}\sum_{i=1}^n h(a_t^{(i)})$, where $h(a):=H(Y\mid A_t=a, S_t=s_t)\in[0,R]$ and $R:=\log|\mathcal{Y}|$ for a finite answer set $\mathcal{Y}$. This estimator is idealized, since $h$ is not observed directly; its sample size $n$ plays the role of the per-step expansion budget that UAA modulates.

\begin{figure*}[!t]
    \centering
        \includegraphics[trim=1 160 110 10, clip, width=\linewidth]{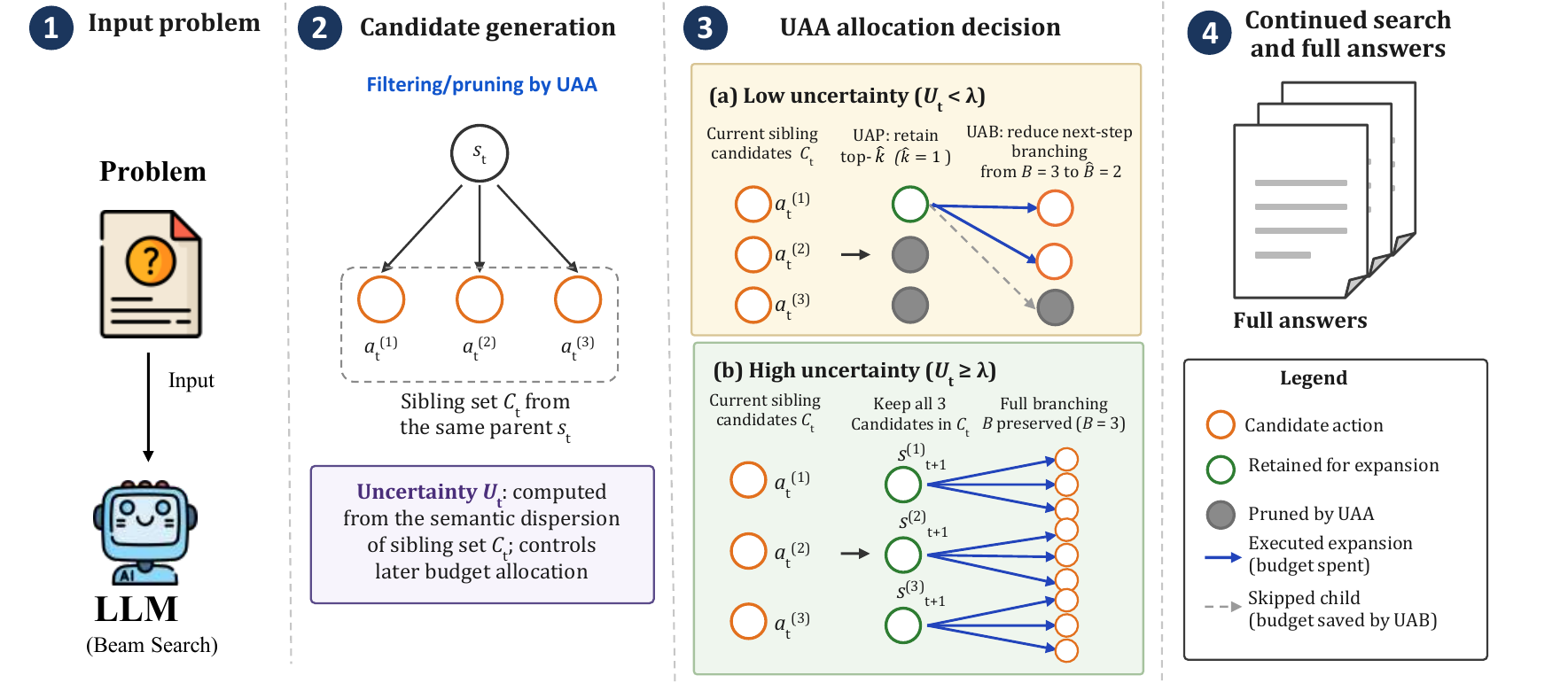}
    \caption{Illustration of UAA integrated with Beam Search, which keeps the top-3 actions per step and expands each into 3 children. At low-uncertainty steps, UAP keeps only the top-$\widehat{k}$ candidates by value and UAB reduces the number of children expanded from each to $\widehat{B}$; at high-uncertainty steps, all candidates and the full expansion budget are preserved. Solid orange arrows denote expansions that are performed (budget spent) and dotted arrows those skipped by UAB (budget saved).}
    \label{fig:uaa_framework}
\end{figure*}

\paragraph{Step uncertainty.} Since textual actions have no natural variance, we measure their dispersion via an embedding $\phi:\mathcal{A}\to\mathbb{R}^d$ with $\|\phi(a)\|=1$, defining the step uncertainty as:
\begin{equation}
\label{eq:step-uncertainty-def}
\begin{aligned}
U_t &:= \Var\!\big(\phi(A_t)\mid s_t\big) \\
&= \mathbb{E}\!\big[\|\phi(A_t)-\mathbb{E}[\phi(A_t)\mid s_t]\|^2\,\big|\,s_t\big].
\end{aligned}
\end{equation}

\begin{assumption}[Lipschitz residual entropy]\label{ass:lipschitz}
There exists $\tilde{h}:\mathbb{R}^d\to[0,R]$ with $h(a)=\tilde{h}(\phi(a))$ for all $a\in\mathcal{A}$, and $\tilde{h}$ is $L$-Lipschitz: $|\tilde{h}(u)-\tilde{h}(v)|\le L\|u-v\|$.
\end{assumption}

\begin{proposition}[Concentration of the entropy estimate]\label{prop:mc-bound}
Let $\sigma^2:=\Var(h(A_t)\mid s_t)\le R^2/4$. Then for any $\varepsilon>0$:
\begin{equation}
\label{eq:bernstein}
\begin{aligned}
&\mathbb{P}\!\left(\big|\widehat{H}_n - H(Y\mid A_t,s_t)\big|\ge\varepsilon\right) \\
&\qquad \le 2\exp\!\left(-\frac{n\,\varepsilon^2}{2\sigma^2+\frac{2}{3}R\,\varepsilon}\right).
\end{aligned}
\end{equation}
\end{proposition}

\noindent Under Assumption~\ref{ass:lipschitz}, $\sigma^2\le L^2 U_t$ (Appendix~\ref{appendix:proof}), so low dispersion implies small $\sigma^2$ and hence faster concentration. Inverting \Eqref{eq:bernstein} with $\sigma^2\le\min\{L^2U_t,\,R^2/4\}$ yields:

\begin{corollary}[Sufficient budget at low uncertainty]\label{cor:uncertainty-sample}
Fix $\varepsilon>0$, $\delta\in(0,1)$. Under Assumption~\ref{ass:lipschitz}, $\mathbb{P}(|\widehat{H}_n-H(Y\mid A_t,s_t)|\ge\varepsilon)\le\delta$ holds whenever
\begin{equation}
\label{eq:sufficient-n}
n \ge \left(\frac{2\min\{L^2U_t,\,R^2/4\}}{\varepsilon^2}+\frac{2R}{3\varepsilon}\right)\ln\frac{2}{\delta}.
\end{equation}
As $U_t\to 0$, the sufficient budget shrinks to $n\ge\tfrac{2R}{3\varepsilon}\ln\tfrac{2}{\delta}$.
\end{corollary}

Corollary~\ref{cor:uncertainty-sample} suggests a simple guideline: when the candidate actions at a step are close to one another, a smaller expansion budget is enough, so the full budget can be reserved for diverse steps. We use it as a guideline rather than a formal guarantee; Appendix~\ref{appendix:proof} states the assumptions behind it, including how it carries over from a step to its children. We next turn this guideline into a concrete allocation rule.

\section{Uncertainty-Aware Allocation}
\label{sec:method}

Building on the observation that low-uncertainty steps yield semantically redundant continuations, we propose \textbf{Uncertainty-Aware Allocation} (UAA), 
a training-free framework that spares the expansion budget at low-uncertainty steps, where siblings are semantically redundant, while preserving the full budget where they are diverse. Specifically, UAA separates value from uncertainty: 
the value function decides \emph{which} candidates are promising, while uncertainty decides \emph{how many} candidates are worth expanding.
UAA realizes this through two complementary mechanisms, \textit{Uncertainty-Aware Pruning} (UAP) and \textit{Uncertainty-Aware Budgeting} (UAB).

Formally, given a reasoning step $s_t$ to be expanded, the policy $\pi_\theta(\cdot\mid s_t)$ first proposes $B_t$ candidate actions 
$\mathcal{C}_t=\{a_t^{(1)},\dots,a_t^{(B_t)}\}\subseteq\mathcal{A}$. 
The step uncertainty $U_t := U(A_t\mid S_t{=}s_t)$ is estimated 
by a lightweight proxy $\widehat{U}_t$ computed from these 
candidates (defined below). 
Figure~\ref{fig:uaa_framework} illustrates the full pipeline 
instantiated on Beam Search.

\textbf{Uncertainty-Aware Pruning (UAP).}
UAP prunes redundant candidate actions only when the current 
step has low uncertainty. Specifically, when $\widehat{U}_t$ 
is below a threshold $\lambda$, UAP keeps the top-$\widehat{k}$ 
candidates ranked by the original value function 
$V(s_t,a_t^{(i)})$ and removes the remaining low-ranked 
siblings. Otherwise, all candidates are kept to preserve 
reasoning diversity:
\[
\widehat{\mathcal{C}}_t =
\begin{cases}
\arg\text{top-}\widehat{k}_{a_t^{(i)}\in\mathcal{C}_t}
V\left(s_t,a_t^{(i)}\right), & \widehat{U}_t<\lambda,\\[4pt]
\mathcal{C}_t, & \text{otherwise,}
\end{cases}
\]
where $\widehat{k}<B_t$ is the number of kept candidates and 
$\lambda$ is the uncertainty threshold. Note that low 
uncertainty indicates only that the siblings are semantically 
redundant and value scores still determine which candidates 
are preserved.

\textbf{Uncertainty-Aware Budgeting (UAB).}
For each retained action $a_t^{(i)}\in\widehat{\mathcal{C}}_t$, 
UAB assigns a next-step expansion budget that depends only on 
the parent's uncertainty $\widehat{U}_t$:
\begin{equation}
\label{equab}
B_{t+1}
:=
\begin{cases}
\widehat{B}, & \widehat{U}_t<\lambda,\\
B, & \text{otherwise,}
\end{cases}
\end{equation}
where $0<\widehat{B}<B$ and $B$ is the expansion budget of 
the base algorithm. 

UAP and UAB act on complementary 
dimensions: UAP reduces the number of parent nodes that 
remain expandable, while UAB reduces the number of children 
generated from each retained low-uncertainty parent, 
yielding a multiplicative reduction in expansions overall. 

\begin{table*}[!t]
\centering
\setlength{\tabcolsep}{1pt}
\resizebox{\textwidth}{!}{%
\begin{tabular}{cclccccccccc}
\toprule
\multirow{2}{*}{LLM} & \multirow{2}{*}{Verifier} & \multirow{2}{*}{Method}
  & \multicolumn{3}{c}{\textbf{GSM8K}}
  & \multicolumn{3}{c}{\textbf{MATH-500}}
  & \multicolumn{3}{c}{\textbf{AIME24}} \\
\cmidrule(r){4-6} \cmidrule(r){7-9} \cmidrule(r){10-12}
& & & Acc.\,(\%) $\uparrow$ & Tokens\,($k$) $\downarrow$ & Time\,($s$) $\downarrow$
    & Acc.\,(\%) $\uparrow$ & Tokens\,($k$) $\downarrow$ & Time\,($s$) $\downarrow$
    & Acc.\,(\%) $\uparrow$ & Tokens\,($k$) $\downarrow$ & Time\,($s$) $\downarrow$ \\
\midrule
\multirow{10}{*}{\rotatebox{90}{\textit{Llama3.1-8B-Instruct}}}
& \multirow{5}{*}{Conf.}
& LiteSearch
  & 80.14 & 0.71 & 14.97
  & 54.67 & 2.06  & 38.57
  & 5.56  & 5.25  & 92.08 \\
\cmidrule(lr){3-12}
& & Beam Search
  & 89.17 & 3.12  & 53.32
  & 57.00 & 11.90 & 234.88
  & 6.67  & 35.97 & 734.44 \\
& & +UAA
  & \cellcolor{blue!5}88.00 & \cellcolor{blue!5}\textbf{2.58}  & \cellcolor{blue!5}34.23
  & \cellcolor{blue!5}56.40 & \cellcolor{blue!5}\textbf{8.27}  & \cellcolor{blue!5}126.21
  & \cellcolor{blue!5}6.67  & \cellcolor{blue!5}\textbf{20.96} & \cellcolor{blue!5}341.69 \\
\cmidrule(lr){3-12}
& & MCTS
  & 90.14 & 8.22  & 114.75
  & 61.40 & 21.65 & 290.40
  & 7.78  & 53.09 & 1261.83 \\
& & +UAA
  & \cellcolor{blue!5}90.45 & \cellcolor{blue!5}\textbf{5.43}  & \cellcolor{blue!5}74.58
  & \cellcolor{blue!5}61.40 & \cellcolor{blue!5}\textbf{11.02} & \cellcolor{blue!5}175.60
  & \cellcolor{blue!5}6.11  & \cellcolor{blue!5}\textbf{19.02} & \cellcolor{blue!5}414.60 \\
\cmidrule(lr){2-12}
& \multirow{5}{*}{PRM}
& LiteSearch
  & 85.67 & 0.73  & 731.14
  & 58.00 & 2.27  & 42.10
  & 5.56  & 6.54  & 100.55 \\
\cmidrule(lr){3-12}
& & Beam Search
  & 91.33 & 2.84  & 49.57
  & 59.33 & 8.13  & 150.59
  & 8.89  & 20.82 & 411.04 \\
& & +UAA
  & \cellcolor{blue!5}90.00 & \cellcolor{blue!5}\textbf{2.21}  & \cellcolor{blue!5}29.45
  & \cellcolor{blue!5}58.00 & \cellcolor{blue!5}\textbf{7.34}  & \cellcolor{blue!5}112.73
  & \cellcolor{blue!5}6.67  & \cellcolor{blue!5}\textbf{18.13} & \cellcolor{blue!5}299.58 \\
\cmidrule(lr){3-12}
& & MCTS
  & 92.00 & 8.94  & 123.73
  & 63.00 & 17.63 & 255.95
  & 8.89  & 33.46 & 581.19 \\
& & +UAA
  & \cellcolor{blue!5}91.33 & \cellcolor{blue!5}\textbf{6.20}  & \cellcolor{blue!5}86.57
  & \cellcolor{blue!5}65.00 & \cellcolor{blue!5}\textbf{10.93} & \cellcolor{blue!5}172.87
  & \cellcolor{blue!5}8.89  & \cellcolor{blue!5}\textbf{16.41} & \cellcolor{blue!5}369.50 \\
\midrule
\multirow{8}{*}{\rotatebox{90}{\textit{Qwen3-32B}}}
& \multirow{5}{*}{Conf.}
& LiteSearch  & 88.00 & 1.84 & 43.02  & 82.33 & 3.74 & 146.4   & 27.33 & 12.91 & 324.7 \\
  \cmidrule(lr){3-12}
& & Beam Search
  & 98.00 & 3.70  & 81.81
  & 86.67 & 10.05 & 282.77
  & 33.33 & 33.61 & 997.66 \\
& & +UAA
  & \cellcolor{blue!5}99.00 & \cellcolor{blue!5}\textbf{2.90}  & \cellcolor{blue!5}58.35
  & \cellcolor{blue!5}85.00 & \cellcolor{blue!5}\textbf{7.52}  & \cellcolor{blue!5}149.69
  & \cellcolor{blue!5}40.00 & \cellcolor{blue!5}\textbf{24.50} & \cellcolor{blue!5}520.17 \\
\cmidrule(lr){3-12}
& & MCTS
  & 98.00 & 48.45 & 983.83
  & 88.00 & 134.47 & 3029.64
  & 26.67 & 62.77  & 1322.08 \\
& & +UAA
  & \cellcolor{blue!5}99.00 & \cellcolor{blue!5}\textbf{7.01}  & \cellcolor{blue!5}277.60
  & \cellcolor{blue!5}89.00 & \cellcolor{blue!5}\textbf{14.12} & \cellcolor{blue!5}338.19
  & \cellcolor{blue!5}23.33 & \cellcolor{blue!5}\textbf{43.53} & \cellcolor{blue!5}1061.97 \\
\cmidrule(lr){2-12}
& \multirow{5}{*}{PRM}
&  LiteSearch  & 91.40 & 1.02 & 30.97  & 82.00 &4.81 & 124.72   &27.33 & 19.23 & 485.36 \\
  \cmidrule(lr){3-12} 
& & Beam Search
  & 97.50 & 3.96  & 109.68
  & 86.33 & 10.07 & 305.94
  & 33.33 & 35.68 & 1164.62 \\
& & +UAA
  & \cellcolor{blue!5}98.48 & \cellcolor{blue!5}\textbf{2.95}  & \cellcolor{blue!5}60.23
  & \cellcolor{blue!5}86.00 & \cellcolor{blue!5}\textbf{7.78}  & \cellcolor{blue!5}161.57
  & \cellcolor{blue!5}36.67 & \cellcolor{blue!5}\textbf{27.36} & \cellcolor{blue!5}640.69 \\
\cmidrule(lr){3-12}
& & MCTS
  & 98.00 & 16.89 & 543.29
  & 84.00 & 27.81 & 644.42
  & 16.00 & 66.16 & 1537.30 \\
& & +UAA
  & \cellcolor{blue!5}99.00 & \cellcolor{blue!5}\textbf{9.45}  & \cellcolor{blue!5}357.35
  & \cellcolor{blue!5}84.00 & \cellcolor{blue!5}\textbf{18.10} & \cellcolor{blue!5}481.05
  & \cellcolor{blue!5}20.00 & \cellcolor{blue!5}\textbf{45.75} & \cellcolor{blue!5}1210.97 \\
\bottomrule
\end{tabular}}
\caption{Results of Beam Search and MCTS on \textit{Llama3.1-8B-Instruct}
and \textit{Qwen3-32B}.}
\label{tab:main_results}
\end{table*}

\textbf{Uncertainty Approximation.}
Since directly computing $U(A_t \mid S_t = s_t)$ is infeasible, 
we approximate it by sampling multiple candidate actions, 
following prior work~\citep{lin2023generating}. 
Specifically, we define:
\begin{equation}
\label{eq:uncertainty_actions}
\widehat{U}_t := \binom{B_t}{2}^{-1} \sum_{i<j} \big(1 - \cos(e_i, e_j)\big),
\end{equation}
where $e_i$ is the embedding of action $a_t^{(i)}$, 
and $\cos(e_i, e_j)$ denotes the cosine similarity. 
For unit-norm embeddings, $U_t=\mathbb{E}[1-\cos(\phi(A_t),\phi(A_t'))\mid s_t]$ for two independent draws $A_t,A_t'\sim\pi_\theta(\cdot\mid s_t)$, so \Eqref{eq:uncertainty_actions} is an unbiased estimate of the step uncertainty $U_t$ in \Eqref{eq:step-uncertainty-def}: low for near-paraphrases, high for diverse candidates.
In our main experiments, we obtain each action's embedding from the feature vector of its last generated token. At each step, $\widehat{U}_t$ is min-max normalized across all candidate states before being compared to $\lambda$, so $\lambda$ is a relative threshold and a single value can transfer across datasets and model scales. When internal embeddings are unavailable (black-box setting), we can replace the embeddings with the confidence or PRM scores.

\section{Experiments}
\label{sec:experiment}

\subsection{Experiment Setup}
\label{sec:setup}
\textbf{Models and Datasets.}
We evaluate UAA primarily on \textit{Llama3.1-8B-Instruct}~\citep{grattafiori2024llama} and \textit{Qwen3-32B}~\citep{yang2025qwen3}, which represent medium-scale and stronger open-source instruction-tuned LLMs. Inference is conducted with temperature set to 0.8 for \textit{Llama3.1-8B-Instruct} and 0.7 for \textit{Qwen3-32B}; other generation parameters follow default settings.
For mathematical reasoning, we use GSM8K~\citep{cobbe2021training}, MATH-500~\citep{lightman2023let}, and AIME24.

\textbf{Methods.} We apply UAA to Beam Search and MCTS, with implementation details in Appendix~\ref{appendix:experimental_details}.
We also include LiteSearch~\citep{wang2024litesearch}, which improves efficiency via step-aware selection and reduced budgets.

\textbf{Value functions.} We evaluate two types of value functions to guide search. (1) \textit{LLM confidence}, (dubbed Conf.), computed as in \Eqref{eq:confidence} and rescaled to the \([0,1]\) range;
(2) \textit{Process reward model}, (dubbed PRM), fine-tuned from Llama3.1-8B-Instruct using RLHFlow and the Deepseek-PRM-Data~\citep{dong2024rlhf}.

\textbf{Metrics.}
We report three metrics: 
(i) \textit{Accuracy}, computed by exact match after majority voting across reasoning paths. Because UAA prunes and reallocates budget, the number of surviving paths entering the majority vote can differ between the base method and UAA;
(ii) \textit{Average tokens}, serving as a proxy for the computational cost. 
(iii) \textit{Time}, the wall-clock time required to generate the entire reasoning tree.

\begin{figure*}[!t]
    \centering
    \subfigure[Beam Search(Conf.)]{
        \includegraphics[width=0.23\linewidth,trim=20 20 20 20]{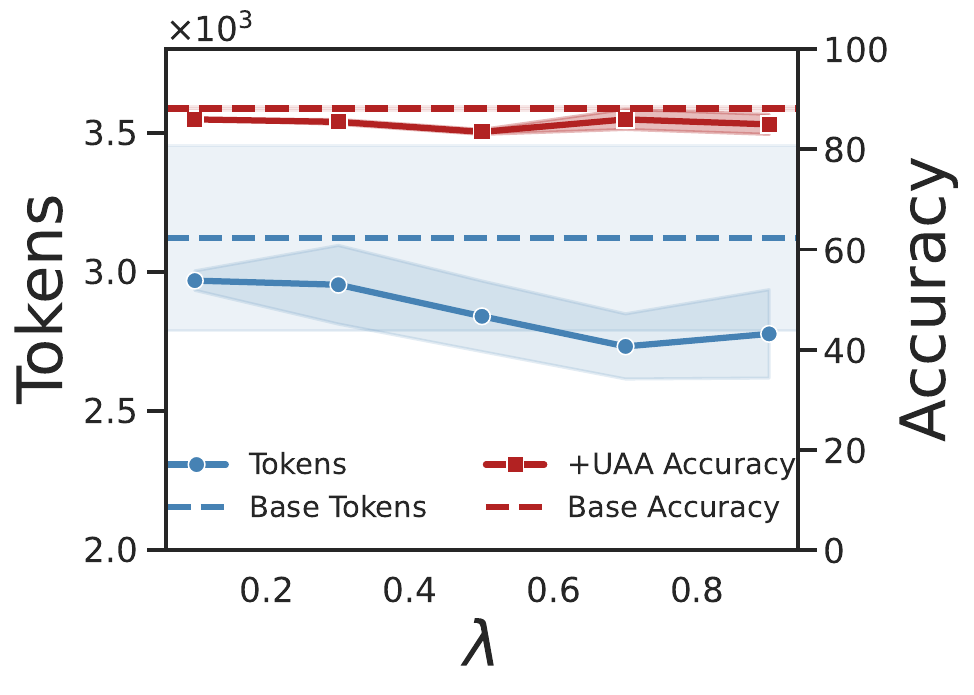}
    }
    \subfigure[Beam Search(PRM)]{
        \includegraphics[width=0.23\linewidth,trim=20 20 20 20]{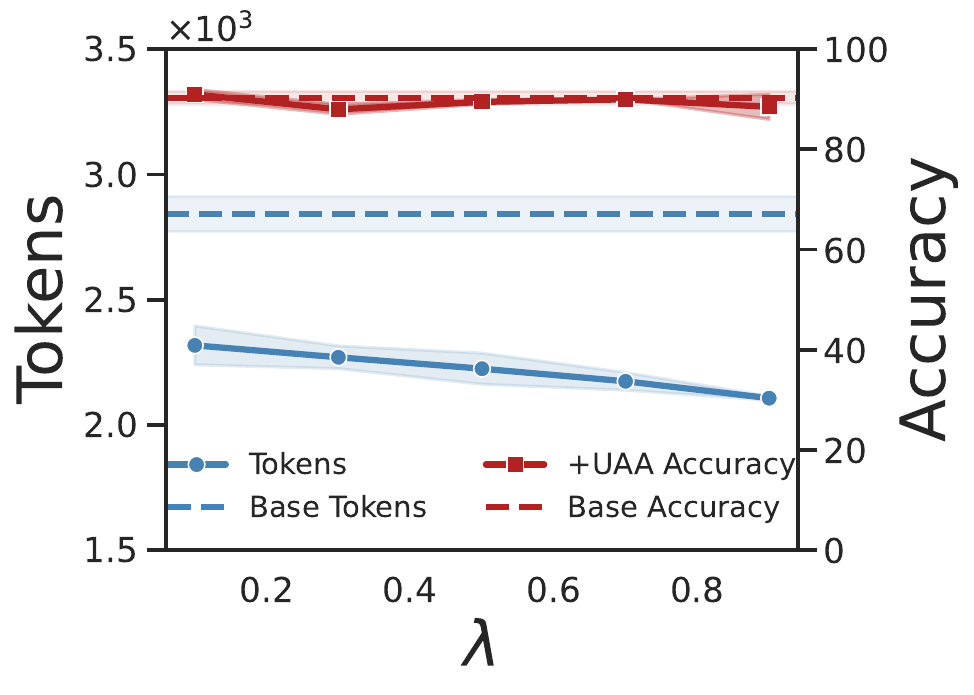}
    }
    \subfigure[MCTS(Conf.)]{
        \includegraphics[width=0.23\linewidth,trim=20 20 20 20]{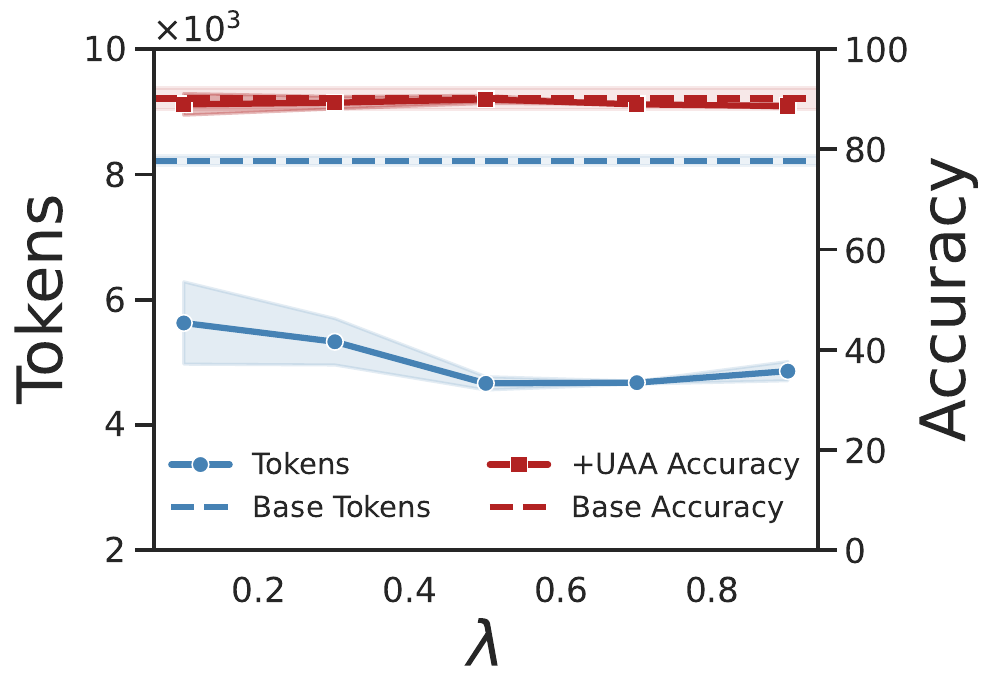}
    }
    \subfigure[MCTS(PRM)]{
        \includegraphics[width=0.23\linewidth,trim=20 20 20 20]{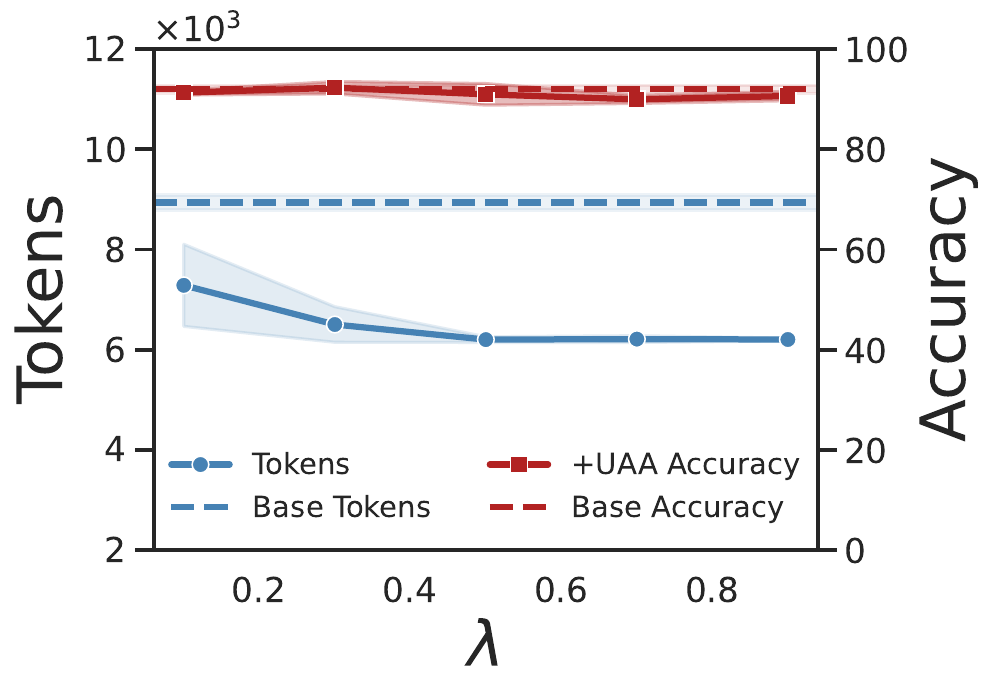}
    }
    \caption{Ablation of $\lambda$ on GSM8K. (a--b) report Beam Search, and (c--d) report MCTS, using model confidence or process reward as the value function. Blue curves show token usage and red curves show accuracy.}
    \label{fig:ablation_thresholds_gsm8k}
\end{figure*}

\subsection{Main results}

\textbf{UAA improves the reasoning efficiency of Beam Search.}
We set the number of beams to $N=16$, the tree width to $B=4$, 
and the maximum tree depth to $40$, with uncertainty threshold 
$\lambda=0.5$, $\widehat{k}=2$, and $\widehat{B}=2$.
As shown in Table~\ref{tab:main_results}, UAA consistently 
reduces token usage and inference time across all datasets 
and verifier settings, while keeping accuracy comparable to 
the base search. On GSM8K with Conf., for instance, token 
usage drops from 3.12k to 2.58k and time from 53.32s to 34.23s with comparable accuracy. UAA generalizes to the stronger \textit{Qwen3-32B} model, reducing token usage by approximately 30\% under both verifiers while slightly improving accuracy, including a +6.7\% gain on AIME24 with Conf. While LiteSearch achieves lower token usage, it consistently sacrifices accuracy, a trade-off that UAA avoids.

\textbf{UAA improves the reasoning efficiency of MCTS.}
We set the tree width to $B=4$, the maximum depth to 40, 
the number of iterations to 100, with $\lambda=0.5$, 
$\widehat{k}=2$, and $\widehat{B}=2$.
UAA consistently reduces token usage and inference time 
while keeping accuracy comparable across all benchmarks. On \textit{Llama3.1-8B-Instruct} with Conf., 
tokens drop from 8.22k to 5.43k and time from 114.75s to 
74.58s on GSM8K. On \textit{Qwen3-32B}, the gains are more 
pronounced: token consumption falls by up to 90\% on 
MATH-500 with Conf., while accuracy stays comparable across most benchmarks.

\textbf{Ablation study.}
Figure~\ref{fig:ablation_thresholds_gsm8k} shows that 
\textbf{increasing $\lambda$ monotonically reduces token usage} while 
accuracy remains essentially flat across all four settings, 
confirming a smooth efficiency--accuracy trade-off. 
Table~\ref{tab:uaa_cores} isolates the contributions of UAP 
and UAB: UAB delivers most of the token savings (2.63k with 
UAB only vs.\ 2.69k with UAP only under Conf.), while UAP 
focuses the remaining budget on high-value trajectories. 
Combining both achieves the best trade-off, with the fewest 
tokens and comparable accuracy under both verifiers. 
A depth analysis is provided in 
Appendix~\ref{appendix:depth_analysis}.


\begin{table}[!t]
\centering
\setlength{\tabcolsep}{2pt}
\resizebox{\columnwidth}{!}{
\begin{tabular}{l c c c c c}
\toprule
Verifier &UAP & UAB & Acc. ($\%$) $\uparrow$ & Tokens ($k$) $\downarrow$ & Time (s) \\
\midrule
 \multirow{4}{*}{Conf.} & \textcolor{red}{\ding{55}} & \textcolor{red}{\ding{55}} & 89.17 & 3.12 & 53.32 \\
 & \textcolor{green!70!black}{\ding{51}} & \textcolor{red}{\ding{55}} & 87.49 & 2.69 & 37.05 \\
& \textcolor{red}{\ding{55}} & \textcolor{green!70!black}{\ding{51}} & 86.42 & 2.63 & 33.69 \\
& \textcolor{green!70!black}{\ding{51}} & \textcolor{green!70!black}{\ding{51}} & 88.00 & \textbf{2.58} &  34.23 \\
\midrule
 \multirow{4}{*}{PRM} & \textcolor{red}{\ding{55}}  & \textcolor{red}{\ding{55}} & 91.33 & 2.84 & 49.57 \\
& \textcolor{green!70!black}{\ding{51}} & \textcolor{red}{\ding{55}} & 90.00 & 2.36 & 34.33 \\
& \textcolor{red}{\ding{55}} & \textcolor{green!70!black}{\ding{51}} & 89.46 & 2.22 & 29.88 \\
& \textcolor{green!70!black}{\ding{51}} & \textcolor{green!70!black}{\ding{51}} &90.00 & \textbf{2.21} &  29.45 \\
\bottomrule
\end{tabular}}
\caption{Ablation study of UAA on GSM8K.}
\label{tab:uaa_cores}
\end{table}

\begin{table}[!t]
\centering
\setlength{\tabcolsep}{1pt}
\resizebox{\columnwidth}{!}{
\begin{tabular}{llcccccc}
\toprule
\multirow{2}{*}{Dataset} & \multirow{2}{*}{Method}
  & \multicolumn{3}{c}{\textit{Llama3.1-8B-Instruct}}
  & \multicolumn{3}{c}{\textit{Qwen3-32B}} \\
\cmidrule(r){3-5} \cmidrule(r){6-8}
  & & Acc.\,(\%) $\uparrow$ & Tokens\,($k$) $\downarrow$ & Time\,($s$) $\downarrow$
    & Acc.\,(\%) $\uparrow$ & Tokens\,($k$) $\downarrow$ & Time\,($s$) $\downarrow$ \\
\midrule
\multirow{2}{*}{StrategyQA}
  & Beam Search
  & 60.67 & 3.23 & 199.97
  & 64.33 & 2.02 & 74.54 \\
  & +UAA
  & \cellcolor{blue!5}65.33 & \cellcolor{blue!5}\textbf{2.92} & \cellcolor{blue!5}195.54
  & \cellcolor{blue!5}72.00 & \cellcolor{blue!5}\textbf{1.64} & \cellcolor{blue!5}70.12 \\
\midrule
\multirow{2}{*}{Blocksworld}
  & Beam Search
  & 18.92 & 1.09 & 130.67
  & 62.16 & 0.64 & 1851.54 \\
  & +UAA
  & \cellcolor{blue!5}24.32 & \cellcolor{blue!5}\textbf{0.93} & \cellcolor{blue!5}108.11
  & \cellcolor{blue!5}67.57
  & \cellcolor{blue!5}\textbf{0.50} 
  & \cellcolor{blue!5}1450.58 \\
\bottomrule
\end{tabular}}
\caption{Results of Beam Search on \textit{Llama3.1-8B-Instruct} 
and \textit{Qwen3-32B}.}
\label{tab:main_results_strategyqa_blocksworld}
\end{table}

\begin{table*}[!t]
\centering

\setlength{\tabcolsep}{1pt}
\resizebox{\textwidth}{!}{%
\begin{tabular}{lccccccccc}
\toprule
\multirow{2}{*}{Method}
  & \multicolumn{3}{c}{\textbf{GSM8K}}
  & \multicolumn{3}{c}{\textbf{MATH-500}}
  & \multicolumn{3}{c}{\textbf{AIME24}} \\
\cmidrule(r){2-4} \cmidrule(r){5-7} \cmidrule(r){8-10}
  & Acc.\,(\%) $\uparrow$ & Tokens\,($k$) $\downarrow$ & Time\,($s$) $\downarrow$
  & Acc.\,(\%) $\uparrow$ & Tokens\,($k$) $\downarrow$ & Time\,($s$) $\downarrow$
  & Acc.\,(\%) $\uparrow$ & Tokens\,($k$) $\downarrow$ & Time\,($s$) $\downarrow$ \\
\midrule
Beam Search
  & 94.00 & 17.18 & 58.2
  & 68.00 & 37.98 & 242.4
  & 13.33 & 71.75 & 712.3 \\
+UAA
  & \cellcolor{blue!5}95.00 & \cellcolor{blue!5}\textbf{12.96} & \cellcolor{blue!5}36.5
  & \cellcolor{blue!5}67.00 & \cellcolor{blue!5}\textbf{22.76} & \cellcolor{blue!5}148.7
  & \cellcolor{blue!5}13.33 & \cellcolor{blue!5}\textbf{37.50} & \cellcolor{blue!5}402.8 \\
\bottomrule
\end{tabular}}
\caption{Results of Beam Search (PRM-guided) on \textit{GPT-4o} 
across GSM8K, MATH-500, and AIME24. }
\label{tab:main_results_gpt}
\end{table*}

\begin{table*}[t]
\centering
\setlength{\tabcolsep}{2mm}{
\resizebox{\textwidth}{!}{
\begin{tabular}{ c c|ccc|ccc|ccc}
\toprule
 Verifier & Uncertainty Type & \multicolumn{3}{c|}{GSM8K} & \multicolumn{3}{c|}{MATH-500} & \multicolumn{3}{c}{AIME24} \\
 & & Acc. ($\%$) $\uparrow$ & Tokens ($k$) $\downarrow$ & Time ($s$) $\downarrow$ & Acc. ($\%$) $\uparrow$ & Tokens  ($k$) $\downarrow$ & Time ($s$) $\downarrow$  & Acc. ($\%$) $\uparrow$ & Tokens ($k$) $\downarrow$ & Time ($s$) $\downarrow$  \\
\midrule
 \multirow{3}{*}{Conf.} & Conf. & 85.44 & 2.86 & 38.02 &  56.67 & \textbf{8.01} & 123.43& 5.56 &  \textbf{20.24} & 332.61\\
 & PRM & 87.00 & 2.63 & 36.38 & 54.00 & 8.70 & 135.46 & 6.67 &  21.61 & 360.51\\
 & Embedding & \textbf{88.00} & \textbf{2.58} &  \textbf{34.23}
& \textbf{56.40} &  8.27 &  126.21
&  6.67 & 20.96 &  341.69 \\
\cmidrule(lr){1-11}
\multirow{3}{*}{PRM} & Conf. & 86.00 &  2.73 & 37.67 & 58.33 & 8.61& 142.79 & 6.67 & 18.81 & 318.28\\
&PRM & 87.49 & 2.40& 33.97 & 57.00 &  7.99 & 132.81 & 6.67 & 20.26 & 333.34\\
& Embedding & \textbf{90.00} & \textbf{2.21} &  29.45
& 58.00 &  \textbf{7.34} & \textbf{112.73}
&  6.67  &  \textbf{18.13} & \textbf{299.58} \\ 
\bottomrule
\end{tabular}}}
\caption{Experimental results of Beam Search on \textit{Llama3.1-8B-Instruct} with different uncertainty estimators. Entries under the ``Uncertainty Type'' column include ``Conf.'' and ``PRM,'' which denote uncertainty scores based on the variance of confidence and process reward values, respectively, while ``Embedding'' is derived from the dispersion of action embeddings.}
\label{tab:uncertainty_types}
\end{table*}

\begin{table*}[!t]
\centering
\resizebox{\textwidth}{!}{
\begin{tabular}{l l c c c c c c c c c }
\toprule
\multicolumn{2}{c}{\multirow{2}{*}{Method}} & \multicolumn{3}{c}{\textbf{GSM8K}} & \multicolumn{3}{c}{\textbf{MATH-500}} & \multicolumn{3}{c}{\textbf{AIME24}} \\
\cmidrule(r){3-5} \cmidrule(r){6-8} \cmidrule(r){9-11}
 & & Acc. (\%) $\uparrow$ & Tokens ($k$) $\downarrow$ & Time ($s$) $\downarrow$
   & Acc. (\%) $\uparrow$ & Tokens ($k$) $\downarrow$ & Time ($s$) $\downarrow$
   & Acc. (\%) $\uparrow$ & Tokens ($k$) $\downarrow$ & Time ($s$) $\downarrow$ \\
\midrule
\multirow{2}{*}{DVTS} & Base  & $87.3 \pm 0.9$ & $3.4 \pm 0.2$ & $75.5 \pm 15.2$ & $60.7 \pm 1.7$ & $12.1 \pm 0.5$ & $408.1 \pm 46.0$ & $6.7 \pm 0.0$ & $31.8 \pm 2.5$ & $973.6 \pm 257.5$ \\
& +UAA  & \cellcolor{blue!5}$88.3 \pm 1.2$ & \cellcolor{blue!5}$\mathbf{1.9} \pm 0.0$ & \cellcolor{blue!5}$\mathbf{48.0} \pm 22.8$
       & \cellcolor{blue!5}$59.0 \pm 2.9$ & \cellcolor{blue!5}$\mathbf{4.9} \pm 0.1$ & \cellcolor{blue!5}$99.4 \pm 3.5$
       & \cellcolor{blue!5}$7.8 \pm 4.2$  & \cellcolor{blue!5}$\mathbf{11.3} \pm 0.9$ & \cellcolor{blue!5}$243.1 \pm 29.3$ \\
\midrule
\multirow{2}{*}{REBASE} & Base & $88.7 \pm 1.2$ & $5.1 \pm 0.0$ & $191.8 \pm 37.4$ & $58.3 \pm 1.2$ & $12.4 \pm 0.9$ & $488.1 \pm 88.1$ & $7.8 \pm 1.6$ & $27.2 \pm 1.4$ & $995.6 \pm 58.4$ \\
                        & +UAA & \cellcolor{blue!5}$88.7 \pm 1.7$ & \cellcolor{blue!5}$\mathbf{4.7} \pm 0.1$ & \cellcolor{blue!5}$153.8 \pm 14.9$
                                 & \cellcolor{blue!5}$61.0 \pm 2.9$ & \cellcolor{blue!5}$\mathbf{11.7} \pm 0.6$ & \cellcolor{blue!5}$436.4 \pm 11.4$
                                 & \cellcolor{blue!5}$8.9 \pm 1.6$  & \cellcolor{blue!5}$\mathbf{26.2} \pm 0.8$ & \cellcolor{blue!5}$790.6 \pm 109.8$ \\
\bottomrule
\end{tabular}}

\caption{Results of DVTS, REBASE, and their combinations with UAA on GSM8K, MATH-500, and AIME24 benchmarks with \textit{Llama3.1-8B-Instruct}.}
\label{tab:dvts_and_rebase}
\end{table*}

\section{Discussion}
\label{sec:discussion}
\textbf{Generalization beyond mathematical reasoning.}
We evaluate UAA on StrategyQA~\citep{geva2021did} and Blocksworld, representing logical reasoning and planning, using confidence scores to guide the search. For Blocksworld, a generated plan is counted as correct only if a simulator verifies that executing it from the initial state reaches the goal state. As shown in Table~\ref{tab:main_results_strategyqa_blocksworld}, UAA consistently reduces token consumption while improving accuracy across both datasets and both models. For example, on StrategyQA, Beam Search accuracy improves from 64.33\% to 72.00\% on \textit{Qwen3-32B} with fewer tokens 
(2.02k$\to$1.64k). These results 
confirm that \textbf{UAA generalizes effectively across diverse reasoning tasks}.

\textbf{UAA under black-box settings.}
When embeddings are unavailable, we substitute the variance of confidence or PRM scores as lightweight proxies. As shown in Table~\ref{tab:uncertainty_types}, score-variance proxies serve as practical 
fallbacks, slightly outperforming embeddings. We evaluate UAA on \textit{GPT-4o}  with PRM-variance as the uncertainty proxy. As shown in 
Table~\ref{tab:main_results_gpt}, UAA reduces token usage 
by about 40\% across all three benchmarks with comparable accuracy, confirming that UAA transfers to closed-source LLMs.

\textbf{UAA improves advanced tree search algorithms.}
We integrate UAA into DVTS~\citep{beeching2024scalingtesttimecompute} 
and REBASE~\citep{wu2024inference}, two methods that 
enhance diversity and reward-based allocation respectively.
As shown in Table~\ref{tab:dvts_and_rebase}, UAA 
consistently reduces token usage and latency while keeping accuracy comparable across all benchmarks, demonstrating that UAA is orthogonal and complementary to existing advanced search strategies.

\section{Conclusion}
\label{sec:conclusion}
We proposed Uncertainty-Aware Allocation (UAA), a simple, training-free framework that improves the efficiency of tree-search reasoning for LLMs. By pruning redundant actions and shrinking budgets at low-uncertainty steps, UAA significantly reduces token usage and inference time while maintaining accuracy across Beam Search, MCTS, and multiple benchmarks. The key insight is that sampling low-uncertainty steps with the same budget as high-uncertainty ones is redundant.
We hope the insights and analyses presented here will inspire further advances in reasoning and training strategies for LLMs. 
A promising direction for future work is to extend this idea to LLM-powered agent planning, where uncertainty-guided budgeting may further improve planning efficiency.

\section*{Limitations}
Our information-theoretic analysis motivates uncertainty-aware allocation under a Lipschitz assumption and an empirical dispersion-persistence premise. Formal guarantees for pruning and a principled choice of UAA's hyperparameters are left for future work. In addition, the best-performing uncertainty estimator requires hidden-state embeddings, so for closed-source models we fall back to score-variance proxies, which we test only on GPT-4o.

\bibliography{custom}

\clearpage

\appendix

\section{Related work}
\paragraph{Efficient tree-search reasoning of LLMs} Tree-search-based reasoning methods~\citep{yao2023tree,feng2023alphazero, wang2024q, chen2024tree} have substantially advanced the reasoning capabilities of LLMs, at a high computational cost. To reduce it, \textit{LiteSearch}~\citep{wang2024litesearch} expands children under a dynamic, step-aware budget, while \textit{FETCH}~\citep{wang2025don} and \textit{EquivPruner}~\citep{liu2025equivpruner} merge or prune semantically equivalent states with additionally trained models. Other methods use uncertainty to guide the search itself~\citep{grosse2024uncertainty,yu2025uncertainty,li2026entropy}, reallocate rollouts across reasoning directions~\citep{wang2025every}, or cut cost through speculative, parallel, or cache-efficient execution and adaptive pruning~\citep{wang2025accelerating,ding2025dynamic,hooper2025ets,xu2025phi}.
In contrast, UAA is \emph{training-free} and \emph{plug-and-play}: it sets the next-level budget from the dispersion of sibling actions, applies to both \emph{black-box} and \emph{white-box} LLMs, and can be layered on top of value-guided methods such as Beam Search, MCTS, DVTS, and REBASE without replacing their selection rules.

\paragraph{Uncertainty estimation in LLMs} 
As LLMs continue to evolve, quantifying the uncertainty in their predictions is critical for enhancing application credibility and mitigating hallucinations. 
A common strategy is to quantify uncertainty from multiple generated responses, offering a means to assess confidence.
Ensemble-based methods, such as variance, consistency, and similarity-based techniques~\citep{fadeeva2023lm, wang2022self, lin2023generating}, quantify uncertainty by analyzing disagreement across sampled outputs. Recent works also apply conformal prediction~\citep{huang2023conformal, cherian2024large, yadkori2024mitigating} to provide set-valued predictions with formal coverage guarantees. Other studies, like Uncertainty-of-Thoughts~\citep{hu2024uncertainty}, leverage uncertainty to guide information-seeking behavior during reasoning.
Inspired by predictive uncertainty quantification, we use uncertainty to allocate test-time search computation. Unlike confidence-only early stopping or value-thresholding methods, UAA keeps value estimation and uncertainty estimation separate: values select promising actions, while uncertainty controls the expansion budget spent on semantically redundant siblings.

\section{Details of pilot study}
\label{appendix:motivation}

In this section, we provide the experimental details of computing the pairwise similarity among sampled continuations. Specifically, for each state \(s_t\), we embed the \(m\) sampled continuations \(\{p^c_{s_t,i}\}_{i=1}^m\) using a pretrained text encoder to obtain vectors \(\{e_i\}_{i=1}^m\). To quantify the diversity of final answers, we simply count the number of distinct solutions among the produced answers.

\subsection{Computation of text similarity}
\label{appendix:text_similarity}
Let the candidate set of paths at decoding step $t$ be denoted as $\{p^c_{s_t,1}, p^c_{s_t,2}, \ldots, p^c_{s_t,m}\}$, where each path is represented by an embedding vector $e(p^c_{s_t,i}) \in \mathbb{R}^d$. We organize these embeddings into a matrix:
\[
E = \begin{bmatrix}
e(p^c_{s_t,1}) \\
e(p^c_{s_t,2}) \\
\vdots \\
e(p^c_{s_t,m})
\end{bmatrix} \in \mathbb{R}^{m \times d}
\]

To quantify the overall similarity among these candidate embeddings, we compute the average pairwise cosine similarity as follows:

\begin{enumerate}
  \item Normalize each embedding vector to unit length to obtain $\tilde{E} \in \mathbb{R}^{m \times d}$:
  \[
  \tilde{e}_i = \frac{e(p^c_{s_t,i})}{\|e(p^c_{s_t,i})\|_2 + \varepsilon}
  \]
  where $\varepsilon$ is a small constant for numerical stability.

  \item Compute the cosine similarity matrix:
  \[
  \mathcal{E} = \tilde{E} \tilde{E}^\top \in \mathbb{R}^{m \times m}
  \]
  where $\mathcal{E}_{ij}$ denotes the cosine similarity between embeddings $i$ and $j$.

  \item Exclude the diagonal (self-similarity) entries and compute the average of all off-diagonal elements:
  \[
  \begin{aligned}
  &\text{Sim}(p^c_{s_t,1}, p^c_{s_t,2}, \dots, p^c_{s_t, m}) \\
  &\qquad = \frac{1}{m(m-1)} \sum_{\substack{i, j = 1 \\ i \ne j}}^m \mathcal{E}_{ij}
  \end{aligned}
  \]
\end{enumerate}

This metric reflects the internal consistency among the candidate paths. A higher value indicates that the candidates are more similar to each other in the embedding space.

\subsection{Details of motivation experiment}
\label{appendix:details_motivation}

We generate reasoning trajectories using a reward-guided beam search strategy~\citep{yao2023tree}. Experiments are conducted on GSM8K benchmark~\citep{cobbe2021training}, using the LLaMA3.1-8B-Instruct model~\citep{grattafiori2024llama}. At each decoding step, 20 candidate continuations are sampled from every active node, and the top 20 are selected based on reward scores provided by the Skywork-o1-Open-PRM-Qwen-2.5-1.5B model\footnote{\url{https://huggingface.co/Skywork/Skywork-o1-Open-PRM-Qwen-2.5-1.5B}}. The beam width is fixed at 20 throughout the search.
This setup yields a wide and diverse reasoning tree, enabling fine-grained analysis of reasoning dynamics. To compute semantic uncertainty, we embed each sampled continuation using the gte-Qwen2-7B-instruct encoder~\citep{li2023towards} and estimate the embedding dispersion by one minus the semantic similarity.

\subsection{Empirical analysis between current reasoning steps and continuations }
\label{appendix:current_step}

Following the previous pilot study, we further investigate the relationship between a reasoning step and the semantic redundancy of its downstream continuations. Specifically, we examine whether the semantic similarity among sibling nodes at a given step correlates with the similarity of their respective continuations.

As shown in Figure~\ref{fig:sibling_nodes}, the distribution of Spearman correlations across multiple steps reveals a consistently positive correlation between the similarity of sibling nodes and the similarity of their continuations. This empirical finding confirms that \textbf{semantic redundancy at the current step is predictive of redundancy in downstream reasoning paths}.

These results provide further motivation for our proposed approach: \textbf{by identifying low-uncertainty steps in the search tree, we can safely reduce the search budget allocated to their continuations without sacrificing solution diversity.}

\begin{figure}[h]
    \centering
    \subfigure[]{
        \includegraphics[width=0.95\linewidth]{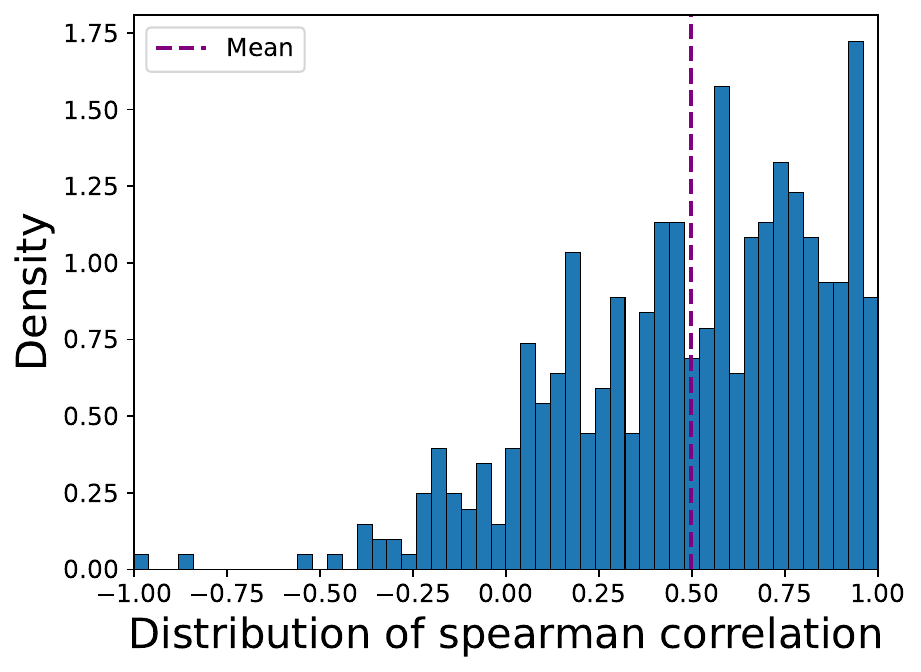}
    }
    \caption{Spearman correlation between the dispersion of continuations and the dispersion of sibling nodes.}
    \label{fig:sibling_nodes}
\end{figure}

\section{More empirical analysis}
\label{appendix:more_empirical_analysis}
In this part, we further investigate the generality of the proposed assumption across MCTS algorithms, large LLMs, and open-domain tasks, i.e, StrategyQA~\citep{geva2021did} . Experimental settings not explicitly specified here follow those in Appendix~\ref{appendix:details_motivation}.

\textbf{Empirical study on an open-domain task} In this experiment, we evaluate Beam Search on the open-domain reasoning task, i.e., StrategyQA, using \textit{LLaMA3.1-8B-Instruct} and \textit{Qwen3-32B} under the same decoding configuration. Since PRM are not applicable to StrategyQA, we adopt LLM confidence as the guidance signal during reasoning. The empirical results are shown in Figures~\ref{fig:motivation_llama_beam_strategyqa} and \ref{fig:motivation_qwen3_beam_strategyqa}. As illustrated, the reasoning trajectories across beams exhibit extremely high textual similarity, with most similarities exceeding 0.95, indicating that beam exploration collapses to nearly identical reasoning patterns. Furthermore, trajectory similarity is negatively correlated with the diversity of final answers (Spearman), indicating that highly overlapping reasoning paths converge to fewer distinct answers and thus add little answer diversity. These observations demonstrate that, on open-domain tasks like StrategyQA, expanding the search space via Beam Search offers limited semantic diversification, and the resulting trajectory consistency does not translate into stable answer outputs.

\begin{figure}[t]
    \centering
    \subfigure[Similarity of reasoning trajectories]{
        \includegraphics[width=0.4\linewidth]{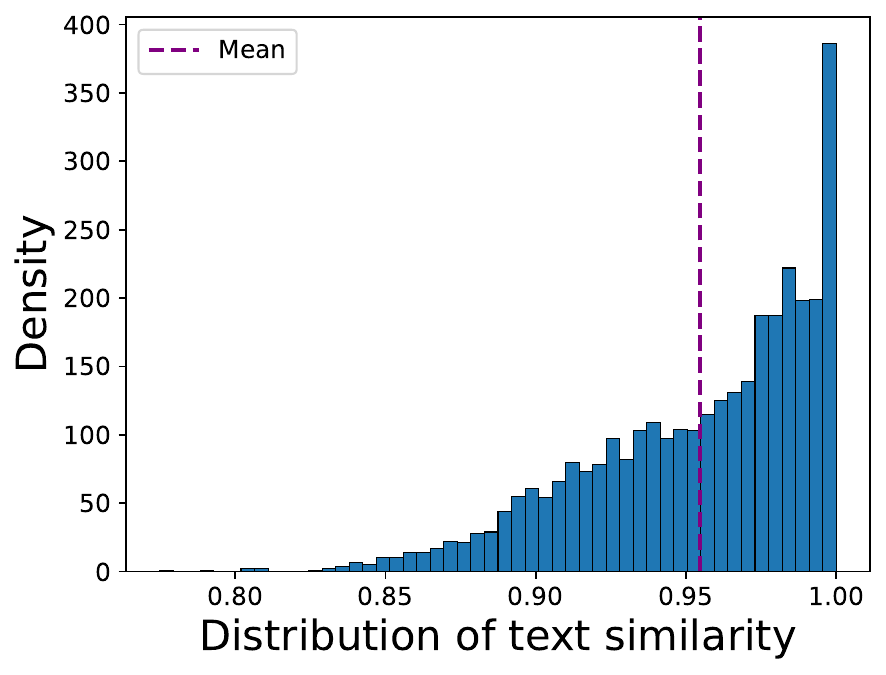}
    }
    \subfigure[Spearman correlation between similarities of reasoning trajectories and diversities of final answers]{
        \includegraphics[width=0.4\linewidth]{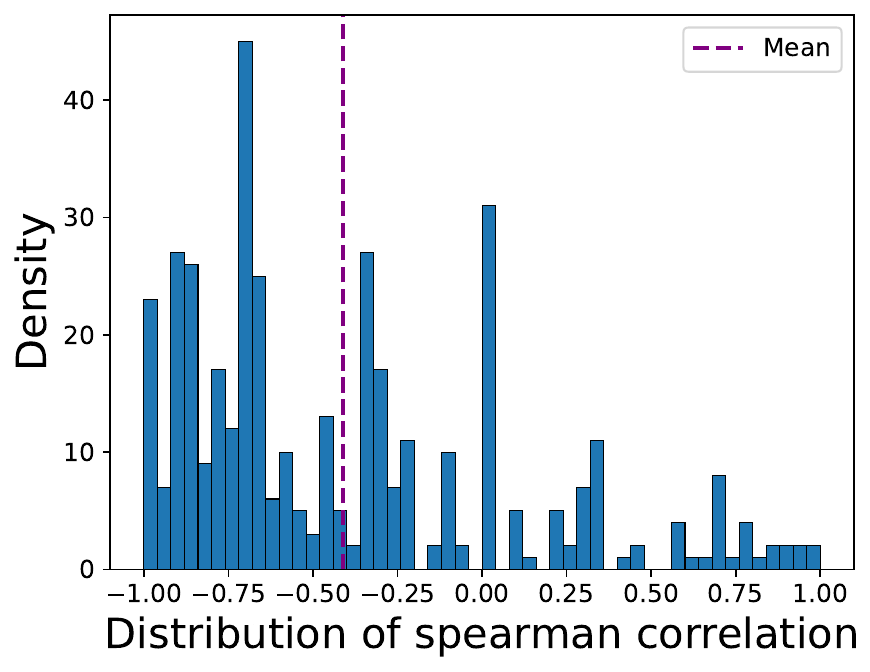}
        }
    \caption{Empirical study of Beam Search on StrategyQA Using \textit{LLaMA3.1-8B-Instruct}.}
    \label{fig:motivation_llama_beam_strategyqa}
\end{figure}

\begin{figure}[t]
    \centering
    \subfigure[Similarity of reasoning trajectories]{
        \includegraphics[width=0.4\linewidth]{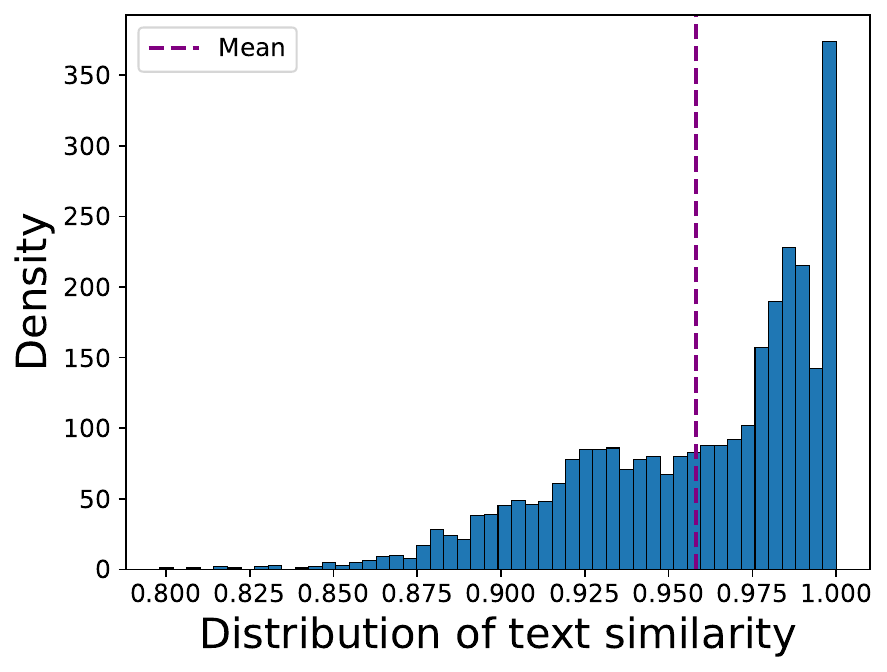}
    }
    \subfigure[Spearman correlation between similarities of reasoning trajectories and diversities of final answers]{
        \includegraphics[width=0.4\linewidth]{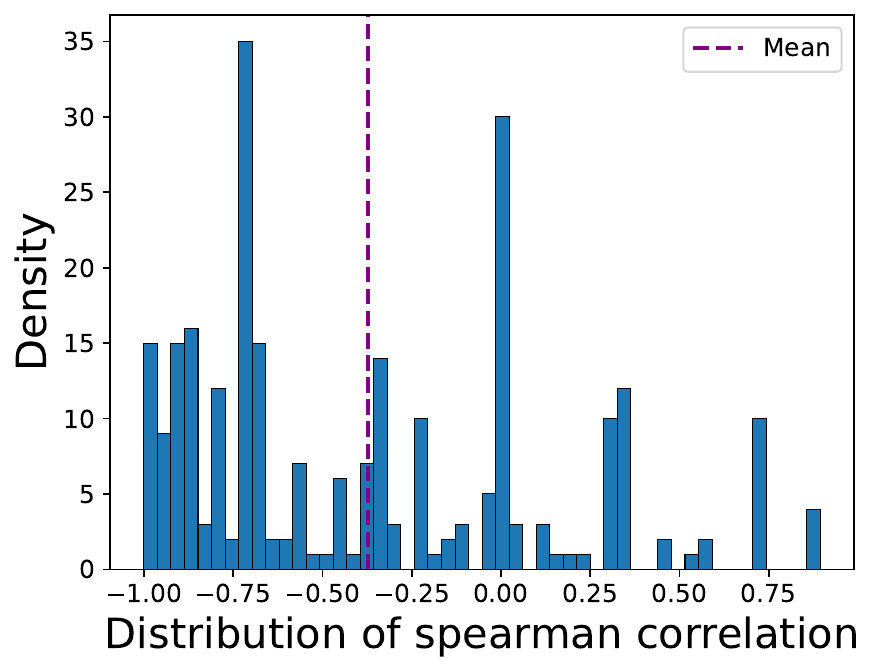}
        }
    \caption{Empirical study of Beam Search on StrategyQA Using \textit{Qwen3-32B}.}
    \label{fig:motivation_qwen3_beam_strategyqa}
\end{figure}

\textbf{Empirical study on MCTS} We conduct a pilot study of MCTS on the GSM8K dataset using the \textit{LLaMA3.1-8B-Instruct} and \textit{Qwen3-32B} models. Specifically, we set the tree width to $B=4$, the maximum depth to 40, and the number of iterations to 100. The results are presented in Figure~\ref{fig:motivation_gsm8k_mcts_llama} and \ref{fig:motivation_gsm8k_mcts_qwen3}. As shown, we observe substantial redundancy in the generated reasoning trajectories. In particular, reasoning steps that are semantically similar contribute little to expanding the coverage of diverse answer candidates.

\begin{figure}[t]
    \centering
    \subfigure[Similarity of reasoning trajectories]{
        \includegraphics[width=0.4\linewidth]{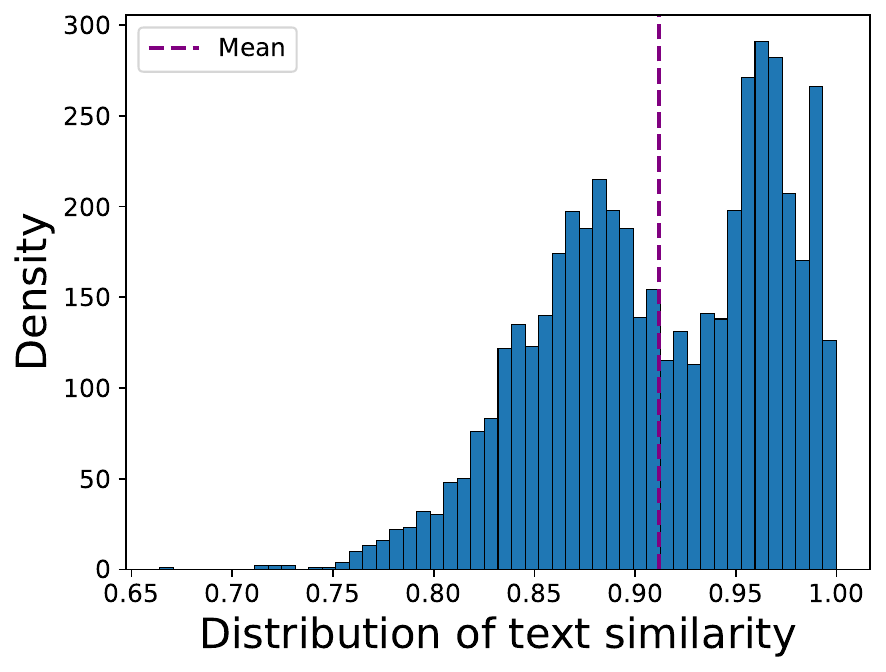}
    }
    \subfigure[Spearman correlation between similarities of reasoning trajectories and diversities of final answers]{
        \includegraphics[width=0.4\linewidth]{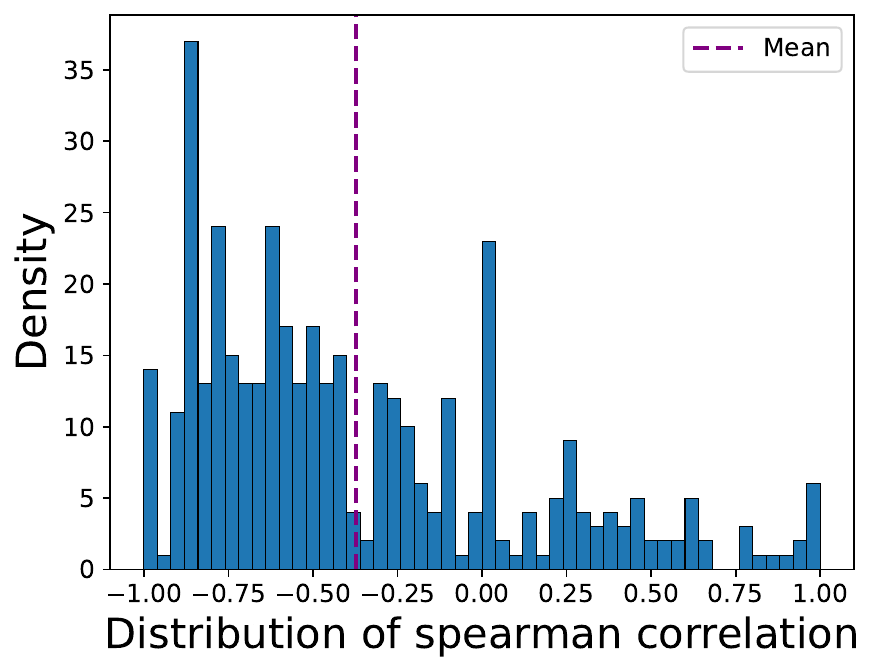}
        }
    \caption{Empirical study of MCTS on GSM8K Using \textit{LLaMA3.1-8B-Instruct}.}
    \label{fig:motivation_gsm8k_mcts_llama}
\end{figure}

\begin{figure}[t]
    \centering
    \subfigure[Similarity of reasoning trajectories]{
        \includegraphics[width=0.4\linewidth]{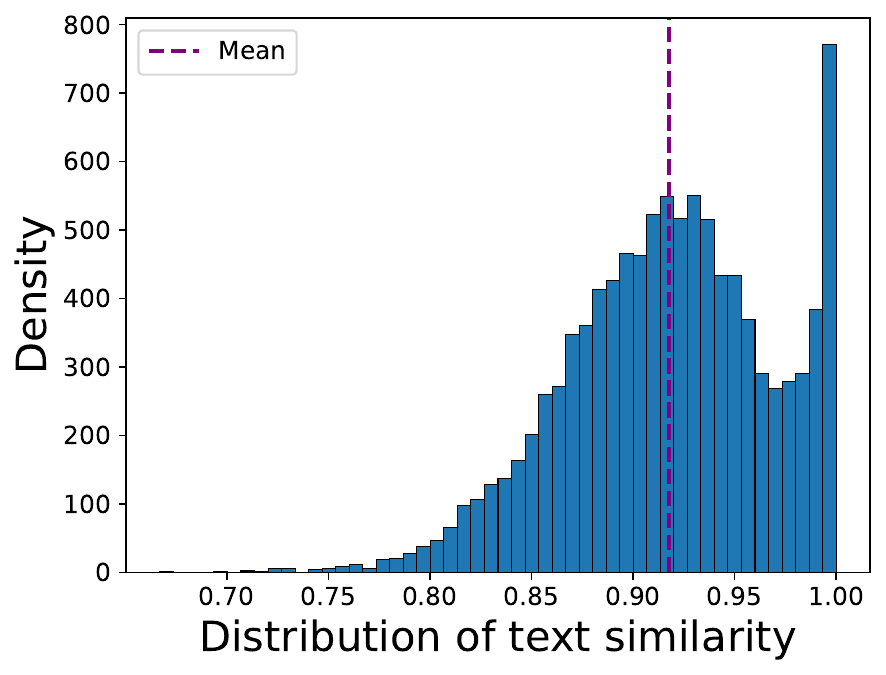}
    }
    \subfigure[Spearman correlation between similarities of reasoning trajectories and diversities of final answers]{
        \includegraphics[width=0.4\linewidth]{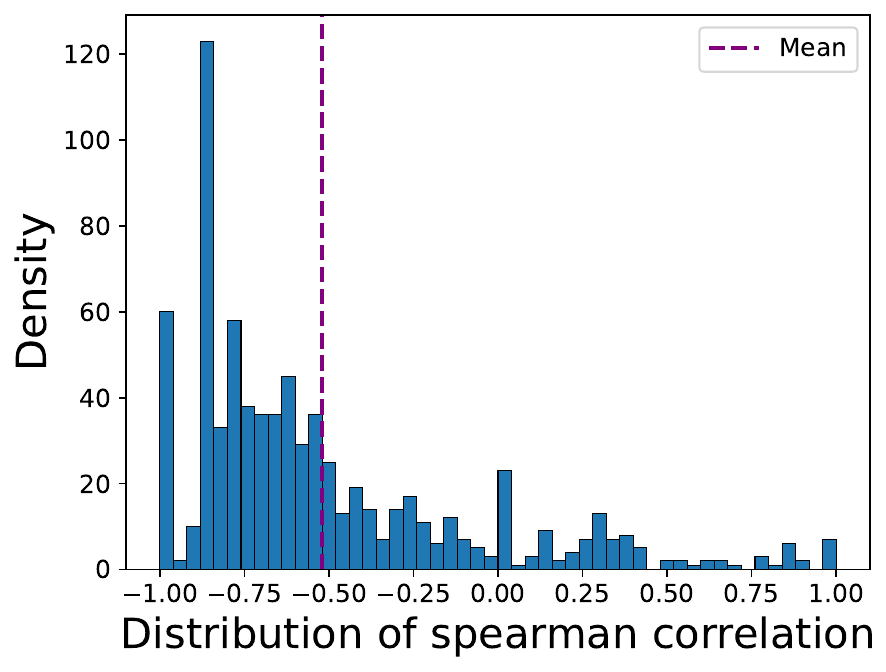}
        }
    \caption{Empirical study of MCTS on GSM8K Using \textit{Qwen3-32B}.}
    \label{fig:motivation_gsm8k_mcts_qwen3}
\end{figure}

\textbf{Empirical study for Larger LLM} We conduct a pilot investigation of both MCTS and Beam Search on GSM8K using the \textit{Qwen3-32B} model. The corresponding results are illustrated in Figure~\ref{fig:motivation_gsm8k_mcts_qwen3} and Figure~\ref{fig:motivation_gsm8k_beam_qwen3}. Across both search strategies, we observe a pronounced redundancy within the generated reasoning trajectories: the text-similarity distributions are heavily skewed toward near-identical reasoning paths, and the Spearman correlations indicate that such similarity contributes minimally to diversifying final predictions. Similar results can also be found in the empirical study of StrategyQA on Qwen3-32B, shown in Figure~\ref{fig:motivation_qwen3_beam_strategyqa}. These findings suggest that even with larger LLMs, the search process tends to collapse into highly similar reasoning branches, limiting its ability to explore meaningfully distinct solution candidates.

\begin{figure}[t]
    \centering
    \subfigure[Similarity of reasoning trajectories]{
        \includegraphics[width=0.4\linewidth]{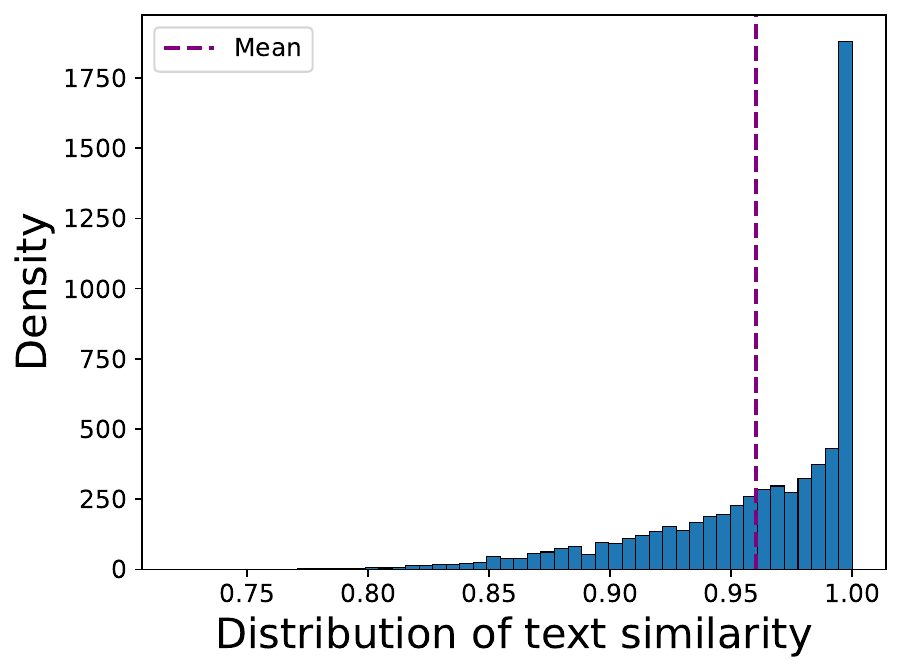}
    }
    \subfigure[Spearman correlation between similarities of reasoning trajectories and diversities of final answers]{
        \includegraphics[width=0.4\linewidth]{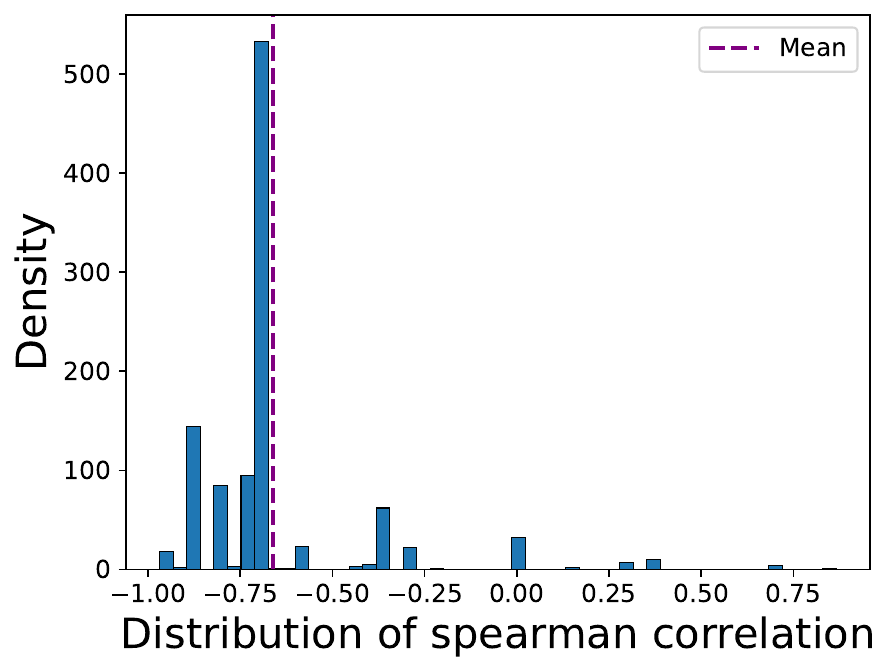}
        }
    \caption{Empirical study of Beam Search on GSM8K Using \textit{Qwen3-32B}.}
    \label{fig:motivation_gsm8k_beam_qwen3}
\end{figure}

\section{Proofs of theoretical analysis}
\label{appendix:proof}

We collect here the assumptions and the proof of  Proposition~\ref{prop:mc-bound}, and Corollary~\ref{cor:uncertainty-sample}. Throughout, $Y$ is the model's predicted final answer (Section~\ref{sec:theoretical_analysis}), random through the policy's sampling with $q$ and $s_t$ fixed; it takes values in a finite set $\mathcal{Y}$, $R:=\log|\mathcal{Y}|$, and $\phi:\mathcal{A}\to\mathbb{R}^d$ is the fixed unit-norm embedding of \Eqref{eq:step-uncertainty-def}.

\paragraph{Assumptions.}
\begin{itemize}
\item[(A1)] \emph{Finite answer space.} $\mathcal{Y}$ is finite, so $h(a)=H(Y\mid A_t=a,s_t)\in[0,R]$ with $R=\log|\mathcal{Y}|$.
\item[(A2)] \emph{I.i.d.\ expansions.} The candidate actions $\{a_t^{(i)}\}_{i=1}^n$ at state $s_t$ are drawn i.i.d.\ from the policy $\pi_\theta(\cdot\mid s_t)$.
\item[(A3)] \emph{Unit-norm embedding.} $\|\phi(a)\|=1$ for every $a\in\mathcal{A}$, and $U_t:=\Var(\phi(A_t)\mid s_t)$ is the step uncertainty of \Eqref{eq:step-uncertainty-def}.
\item[(A4)] \emph{Lipschitz residual entropy (Assumption~\ref{ass:lipschitz}).} $h(a)=\tilde h(\phi(a))$ with $\tilde h:\mathbb{R}^d\to[0,R]$ that is $L$-Lipschitz in Euclidean distance.
\end{itemize}

\paragraph{Bernstein's inequality (bounded i.i.d.).}
We use the following standard form. Let $Z_1,\dots,Z_n$ be i.i.d.\ with $|Z_i-\mathbb{E}Z_i|\le M$ almost surely and $\Var(Z_i)=\sigma^2$, and let $\bar Z_n=\tfrac1n\sum_i Z_i$. Then for every $\varepsilon>0$,
\begin{equation}
\label{eq:bernstein-classical}
\mathbb{P}\big(|\bar Z_n-\mathbb{E}Z_1|\ge\varepsilon\big)\;\le\;2\exp\!\left(-\frac{n\varepsilon^2}{2\sigma^2+\tfrac{2}{3}M\varepsilon}\right).
\end{equation}

\subsection{Proof of Proposition~\ref{prop:mc-bound}}

\begin{proof}
Set $Z_i:=h(a_t^{(i)})$. By (A2) the $Z_i$ are i.i.d., and by (A1) they lie in $[0,R]$, so $\mathbb{E}Z_1=H(Y\mid A_t,s_t)$ and $|Z_i-\mathbb{E}Z_1|\le R$; thus we may take $M=R$ in \eqref{eq:bernstein-classical}. Writing $\sigma^2:=\Var(h(A_t)\mid s_t)=\Var(Z_1)$ and noting $\bar Z_n=\widehat{H}_n$, \eqref{eq:bernstein-classical} gives exactly
\[
\begin{aligned}
&\mathbb{P}\big(|\widehat{H}_n-H(Y\mid A_t,s_t)|\ge\varepsilon\big) \\
&\qquad \le 2\exp\!\left(-\frac{n\varepsilon^2}{2\sigma^2+\tfrac{2}{3}R\varepsilon}\right).
\end{aligned}
\]
Finally, since $Z_1\in[0,R]$, Popoviciu's inequality on the variance of a bounded random variable gives $\sigma^2\le (R-0)^2/4=R^2/4$.
\end{proof}

\begin{proof}[Proof of Corollary~\ref{cor:uncertainty-sample}]
By Proposition~\ref{prop:mc-bound} it suffices to choose $n$ so that the right-hand side of \eqref{eq:bernstein} is at most $\delta$, i.e.
\(
2\exp\!\big(-n\varepsilon^2/(2\sigma^2+\tfrac23 R\varepsilon)\big)\le\delta.
\)
Taking logarithms and rearranging,
\[
n\;\ge\;\frac{2\sigma^2+\tfrac23 R\varepsilon}{\varepsilon^2}\,\ln\frac{2}{\delta}
\;=\;\left(\frac{2\sigma^2}{\varepsilon^2}+\frac{2R}{3\varepsilon}\right)\ln\frac{2}{\delta}.
\]
Since $\sigma^2\le R^2/4$ by Proposition~\ref{prop:mc-bound} and $\sigma^2\le L^2U_t$ under Assumption~\ref{ass:lipschitz} (shown in the remark below), every $n$ satisfying \eqref{eq:sufficient-n} satisfies this condition. Letting $U_t\to0$ leaves the residual budget $n\ge\tfrac{2R}{3\varepsilon}\ln\tfrac2\delta$.
\end{proof}

\begin{remark}[Variance contraction under a Lipschitz map]
We give the elementary i.i.d.-copy argument behind $\sigma^2\le L^2U_t$. Let $X:=\phi(A_t)$ and let $X'$ be an independent copy of $X$ (both conditioned on $s_t$). For any function $g$, variance admits the i.i.d.-copy representation
\[
\Var(g(X))=\tfrac12\,\mathbb{E}\big[(g(X)-g(X'))^2\big].
\]
Apply this with $g=\tilde h$. By (A4), $|\tilde h(X)-\tilde h(X')|\le L\|X-X'\|$, hence
\begin{align*}
    \Var\!\big(\tilde h(X)\big)&=\tfrac12\,\mathbb{E}\big[(\tilde h(X)-\tilde h(X'))^2\big] \\
\\ &\le \tfrac12 L^2\,\mathbb{E}\big[\|X-X'\|^2\big]
\\&= L^2\,\Var(X),
\end{align*}
where the last equality again uses the i.i.d.-copy identity (coordinatewise) for $X$. Since $\tilde h(X)=h(A_t)$ and $X=\phi(A_t)$, this reads $\sigma^2=\Var(h(A_t)\mid s_t)\le L^2\,\Var(\phi(A_t)\mid s_t)=L^2 U_t$.
\end{remark}

\begin{assumption}[Dispersion persistence]\label{ass:persistence}
Let $U_{t+1}:=\Var(\phi(A_{t+1})\mid s_{t+1})$. There exist $\rho,\eta\ge 0$ such that $U_{t+1}\le\rho\,U_t+\eta$ for every child state $s_{t+1}=[s_t,a_t]$.
\end{assumption}

\begin{remark}[Scope of the analysis]\label{rem:scope}
\emph{(i) From a step to its children.} Corollary~\ref{cor:uncertainty-sample} concerns actions sampled at $s_t$, whereas UAB (\Eqref{equab}) sets the budget of the child states. If Assumption~\ref{ass:lipschitz} holds at $s_{t+1}$ with the same $L$, the argument above gives $\sigma_{t+1}^2\le L^2U_{t+1}$, and Assumption~\ref{ass:persistence} then yields $\sigma_{t+1}^2\le L^2(\rho\,U_t+\eta)$: the sufficient budget at a child shrinks with the parent's dispersion $U_t$, the quantity UAB thresholds. Assumption~\ref{ass:persistence} is an empirical premise rather than a provable property; its support is the positive average correlation between sibling similarity and continuation similarity in Appendix~\ref{appendix:current_step}. Since $U_{t+1}\le 1$ for unit-norm embeddings, the assumption holds trivially with $\eta=1$; it is informative only when $\eta$ is small enough that $\rho\,U_t+\eta<R^2/(4L^2)$, and we do not estimate $\rho$ and $\eta$ here.

\emph{(ii) Size of $L$.} Combining the bound above with Proposition~\ref{prop:mc-bound} gives $\sigma^2\le\min\{L^2U_t,\,R^2/4\}$, so low dispersion reduces the sufficient budget only when $U_t<R^2/(4L^2)$. Assumption~\ref{ass:lipschitz} may fail, or hold only with a large $L$, for steps that differ in a single digit or sign yet lead to different answers. In that case the bound reverts to the worst-case budget, and budget reductions at such steps are not covered by the analysis.

\emph{(iii) Pruning and hyperparameters.} The analysis does not bound the probability that UAP discards a correct branch. Assumption~\ref{ass:lipschitz} only implies $|h(a)-h(a')|\le L\|\phi(a)-\phi(a')\|$, i.e., near-duplicate siblings have nearly equal residual entropy. Likewise, UAA is a two-level discretization of the budget in \Eqref{eq:sufficient-n}; since $L$ is unknown, the analysis does not determine $\lambda$, $\widehat{k}$, or $\widehat{B}$, which we set empirically (Figures~\ref{fig:ablation_thresholds_gsm8k} and~\ref{fig:parameters_sensitivity}).

\emph{(iv) Link to correctness.} $Y$ is the model's predicted answer, so $H(Y\mid A_t,s_t)$ measures the model's own uncertainty rather than correctness. Proposition~\ref{prop:mc-bound} uses only that $h$ is bounded: it holds verbatim with $h$ replaced by the success probability $g(a):=\mathbb{P}(Y=y^\star\mid A_t=a,s_t)\in[0,1]$, where $y^\star$ is the ground-truth answer, and $R=1$. In that case $\widehat{H}_n$ estimates $\mathbb{P}(Y=y^\star\mid s_t)$, the probability that $s_t$ leads to the correct answer, and the same budget argument applies if $g$ satisfies Assumption~\ref{ass:lipschitz}.
\end{remark}

\section{Experimental details}
\label{appendix:experimental_details}

\subsection{Details of tree search algorithms}

\subsubsection{Beam search}
Beam search optimizes the value of reasoning trajectories by searching over its per-step predictions. Our implementation is similar to value-guided Breadth First Search~\citep{feng2023alphazero, yao2023tree, snell2024scaling, beeching2024scalingtesttimecompute}. Concretely, we consider a fixed number of beams $N$ and a beam width $B$. We then run the following steps:

\begin{enumerate}
    \item Sample $N$ initial predictions for the first step in the solution.
    \item Score the generated steps according to the PRM's predicted step-wise reward-to-go estimate (which also corresponds to the total reward from the prefix since the reward is sparse in this setting).
    \item Filter for only the top $\frac{N}{B}$ highest scoring steps.
    \item From each candidate, sample $B$ proposals from the next step, resulting in a total of $\frac{N}{B} \times B$ candidate prefixes again. Then repeat steps 2--4.
\end{enumerate}

\paragraph{Incorporating UAA into Beam Search.}  
We integrate UAA into the above process by modifying Steps 3 and 4. In Step 3, \textit{Uncertainty-Aware Pruning} (UAP) removes redundant candidates based on their uncertainty estimates. In Step 4, \textit{Uncertainty-Aware Budgeting} (UAB) allocates the expansion budget across the retained candidates.  
Moreover, to enhance the stability of UAA across different datasets and scales, we apply min-max normalization to the estimated uncertainty values during each expansion step. Specifically, following the notation in Section~\ref{sec:preliminaries}, let \(\{s_t^{(l)}\}_{l=1}^L\) be the set of candidate states at step \(t\). We compute their corresponding uncertainty estimates 
\begin{align}
&\widehat{U}^\prime(A_t \mid S_t = s_t^{(l)}) = \nonumber \\
&\resizebox{0.86\columnwidth}{!}{$\displaystyle
\frac{
\widehat{U}(A_t \mid S_t = s_t^{(l)}) - \min\limits_l \widehat{U}(A_t \mid S_t = s_t^{(l)})
}{
\max\limits_l \widehat{U}(A_t \mid S_t = s_t^{(l)}) - \min\limits_l \widehat{U}(A_t \mid S_t = s_t^{(l)})
}.
$}
\end{align}

\subsubsection{Monte Carlo Tree Search (MCTS)}
MCTS~\citep{feng2023alphazero,wang2024q} is a best-first search algorithm that incrementally builds a search tree by balancing exploration and exploitation through Monte Carlo simulations. It consists of four main stages:

\begin{enumerate}
    \item \textbf{Selection:} Starting from the root node, the algorithm selects child nodes based on the Upper Confidence Bound (UCB) criterion:
\[
\resizebox{\linewidth}{!}{$\displaystyle
s_{t+1} = \arg\max_{s_j \in \text{Children}(s_t)} \left[ V(s_j) + c \cdot \sqrt{ \frac{\log N(s_t)}{1 + N(s_j)} } \right],
$}
\]
where \(V(s_j)\) is the value of node \(s_j\), \(N(s_t)\) is the visit count of the parent node, and \(c\) is a constant balancing exploration and exploitation. This selection continues until a leaf node is reached. 

    \item \textbf{Expansion:} If the selected leaf node is non-terminal (i.e., no final answer has been generated), it is expanded by decoding one step using the policy model to generate \(B\) child nodes.

    \item \textbf{Simulation:} From each expanded node \(s_t\), the policy model performs \(B\) rollouts to generate complete reasoning trajectories, which are then scored by a reward model. The initial value of \(s_c\) is given by a value function.
    \item \textbf{Backpropagation:} Once the child nodes are evaluated, their values are propagated up the tree. The parent node \(s_t\) updates its value and visit count as follows:
\[
\begin{aligned}
V(s_t) &\leftarrow \frac{N(s_t) \cdot V(s_t) + \sum_{c=1}^B V(s_c)}{N(s_t) + B}, \\
N(s_t) &\leftarrow N(s_t) + 1.
\end{aligned}
\]
\end{enumerate}

\paragraph{Incorporating UAA into MCTS.}  
We integrate UAA into MCTS exactly as defined in Section~\ref{sec:method} (\Eqref{equab}); the value function $V$ and the UCB selection rule are left unchanged. In the Expansion stage, after the policy model generates the candidate actions of a non-terminal leaf, \textit{Uncertainty-Aware Pruning} (UAP) computes the sibling uncertainty $U_t$ (\Eqref{eq:uncertainty_actions}): if $U_t<\lambda$, only the top-$\widehat{k}$ candidates ranked by the original value $V(s_t,a_t^{(i)})$ are retained as expandable children and the remaining low-ranked siblings are dropped; otherwise all candidates are kept. \textit{Uncertainty-Aware Budgeting} (UAB) then sets the number of children to decode for each retained node by the same uncertainty gate: $\widehat{B}$ children when $U_t<\lambda$ and the default $B$ otherwise. Uncertainty thus modulates only the branching factor, never the value ranking or the UCB selection.  

\subsubsection{Diverse Verifier Tree Search (DVTS)}  
DVTS~\citep{beeching2024scalingtesttimecompute} is an extension of beam search designed to maximize diversity at large test-time compute budgets. It follows a similar procedure to beam search, with the following modifications:

\begin{enumerate}
    \item For a given budget \(N\) and beam width \(M\), expand the initial set of beams into \(K\) independent subtrees.
    \item Within each subtree, select the step with the highest score.
    \item From the selected nodes, generate \(M\) new steps and again retain the one with the highest PRM score in each subtree.
    \item Repeat Step~(3) until the EOS token is generated or the maximum tree depth is reached.
\end{enumerate}

\paragraph{Incorporating UAA into DVTS.}  
We incorporate UAA into DVTS by applying the rule of Section~\ref{sec:method} (\Eqref{equab}) independently within each subtree, leaving the per-subtree PRM selection unchanged. In Step~2, \textit{Uncertainty-Aware Pruning} (UAP) computes the sibling uncertainty $U_t$ (\Eqref{eq:uncertainty_actions}) among the candidates of the selected node: when $U_t<\lambda$ (i.e., the siblings are semantically redundant), only the top-$\widehat{k}$ candidates ranked by the original PRM value are kept eligible and the remaining low-ranked siblings are pruned; otherwise all candidates are retained to preserve diversity. In Step~3, \textit{Uncertainty-Aware Budgeting} (UAB) controls the expansion width without enlarging the overall budget: high-uncertainty nodes ($U_t\ge\lambda$) keep the default allocation of $M$ expansions, whereas low-uncertainty nodes ($U_t<\lambda$) use the reduced budget $\widehat{B}<M$.

\subsubsection{Reward Balanced Search (REBASE)}  
REBASE~\citep{wu2024inference} is a tree search algorithm that balances exploration based on reward scores without expensive rollouts. The procedure is as follows:

\begin{enumerate}
    \item \textbf{Initialization:} Given a question \(x\), balance temperature \(T_b > 0\), and target number of solutions \(N\), sample \(N\) initial steps from the policy model to form the first layer of nodes. Set the budget at depth 0 to \(B_0 = N\).
    \item \textbf{Reward assignment:} At depth \(i\), assign rewards \(R(n_j)\) to each node \(n_j\). Count completed solutions \(C_i\), and update the budget:
    \[
    B_i \leftarrow B_{i-1} - C_i.
    \]
    If \(B_i = 0\), terminate and return the \(N\) collected solutions.
    \item \textbf{Balanced expansion:} For each node \(n_j\) at depth \(i\), compute its expansion width:
    \[
    W_j = \text{Round}\!\left(\frac{B_i \exp(R(n_j)/T_b)}{\sum_k \exp(R(n_k)/T_b)}\right).
    \]
    Expand \(W_j\) children for node \(n_j\) by sampling from the policy model.
    \item \textbf{Iteration:} Repeat steps (2)--(3) until all \(N\) solutions are generated or the maximum tree depth is reached.
\end{enumerate}

\paragraph{Incorporating UAA into REBASE.}  
We integrate UAA into REBASE following the rule of Section~\ref{sec:method} (\Eqref{equab}); the reward scores $R(n_j)$ that drive the balanced allocation are left unchanged. Before the reward-based budget of Step~3 is computed, \textit{Uncertainty-Aware Pruning} (UAP) evaluates the sibling uncertainty $U_t$ (\Eqref{eq:uncertainty_actions}) at each parent: when $U_t<\lambda$, only the top-$\widehat{k}$ nodes ranked by the original reward $R(\cdot)$ remain eligible for expansion and the redundant low-ranked siblings are pruned; otherwise all nodes are retained. \textit{Uncertainty-Aware Budgeting} (UAB) then caps the total expansion budget at each parent before it is distributed by the REBASE reward-temperature weighting in Step~3: low-uncertainty parents ($U_t<\lambda$) use the reduced budget $\widehat{B}$ in place of $B_i$, while high-uncertainty parents keep $B_i$; the relative widths $W_j$ across children are still set by the unchanged reward-temperature weighting. Thus, uncertainty rescales only the per-parent budget, not the reward ranking.

\section{Depth analysis of the search trees}
\label{appendix:depth_analysis}

In this section, we present a reasoning-depth analysis for both Beam Search and MCTS, based on the experiments reported in Section~\ref{sec:experiment}. As shown in Table~\ref{tab:depth_analysis}, integrating UAA into Beam Search and MCTS consistently preserves the effective reasoning depth of the search tree. This demonstrates that UAA reduces token consumption without prematurely truncating the exploration process. Rather than discarding meaningful reasoning steps, UAA removes only redundant expansions while keeping the valid reasoning trajectory fully extendable. Consequently, UAA achieves substantially lower token usage while maintaining comparable, and occasionally improved, accuracy across GSM8K, MATH-500, and AIME24. These results confirm that \textbf{UAA enhances efficiency while preserving the reasoning depth of the search tree.}

\begin{table*}[!t]
\centering
\resizebox{0.8\textwidth}{!}{
\begin{tabular}{clcccccc}
\toprule
\multirow{2}{*}{Verifier} & \multirow{2}{*}{Method} & \multicolumn{3}{c}{\textit{Llama3.1-8B-Instruct }} & \multicolumn{3}{c}{\textit{Qwen3-32B}} \\
\cmidrule(r){3-5} \cmidrule(r){6-8}
& & \textbf{GSM8K} &  \textbf{MATH-500}&  \textbf{AIME24} & \textbf{GSM8K} &  \textbf{MATH-500}&  \textbf{AIME24} \\
\midrule
 \multirow{4}{*}{\centering Conf.}
&  Beam Search 
&   5.12  & 9.46 & 17.29
&   13.73 & 23.84 & 38.22\\
&  +UAA 
&  5.22 & 9.01 &  17.14
&  20.69 & 27.38 & 37.98 \\
\cmidrule(lr){2-8}
& MCTS &12.45 & 22.75 & 31.74 &14.25 &25.46 & 39.15 \\
& +UAA & 15.25 & 20.64 & 32.43 & 15.86 & 25.27 & 39.68 \\
\cmidrule(lr){1-8}
\multirow{4}{*}{\centering PRM} &  Beam Search 
& 4.87  & 8.64 & 14.87 &15.24 & 23.83 & 38.28\\
& +UAA 
&  4.93 & 8.86 & 16.02 & 20.40 & 24.80 & 37.70 \\
\cmidrule(lr){2-8}
& MCTS  & 9.51 & 19.37 & 32.91 & 17.25 & 26.5& 39.94 \\
& +UAA & 9.63 & 16.29 & 35.46 &  18.41 &25.49 & 38.12 \\
\bottomrule
\end{tabular}}
\caption{Depth analysis of Beam Search and MCTS on \textit{Llama3.1-8B-Instruct } and \textit{Qwen3-32B} across GSM8K, MATH-500, and AIME24.  }
\label{tab:depth_analysis}
\end{table*}

\begin{figure*}[!t]
    \centering
    \subfigure[Conf.\ --- Acc.\ Gap]{
        \includegraphics[width=0.3\linewidth]{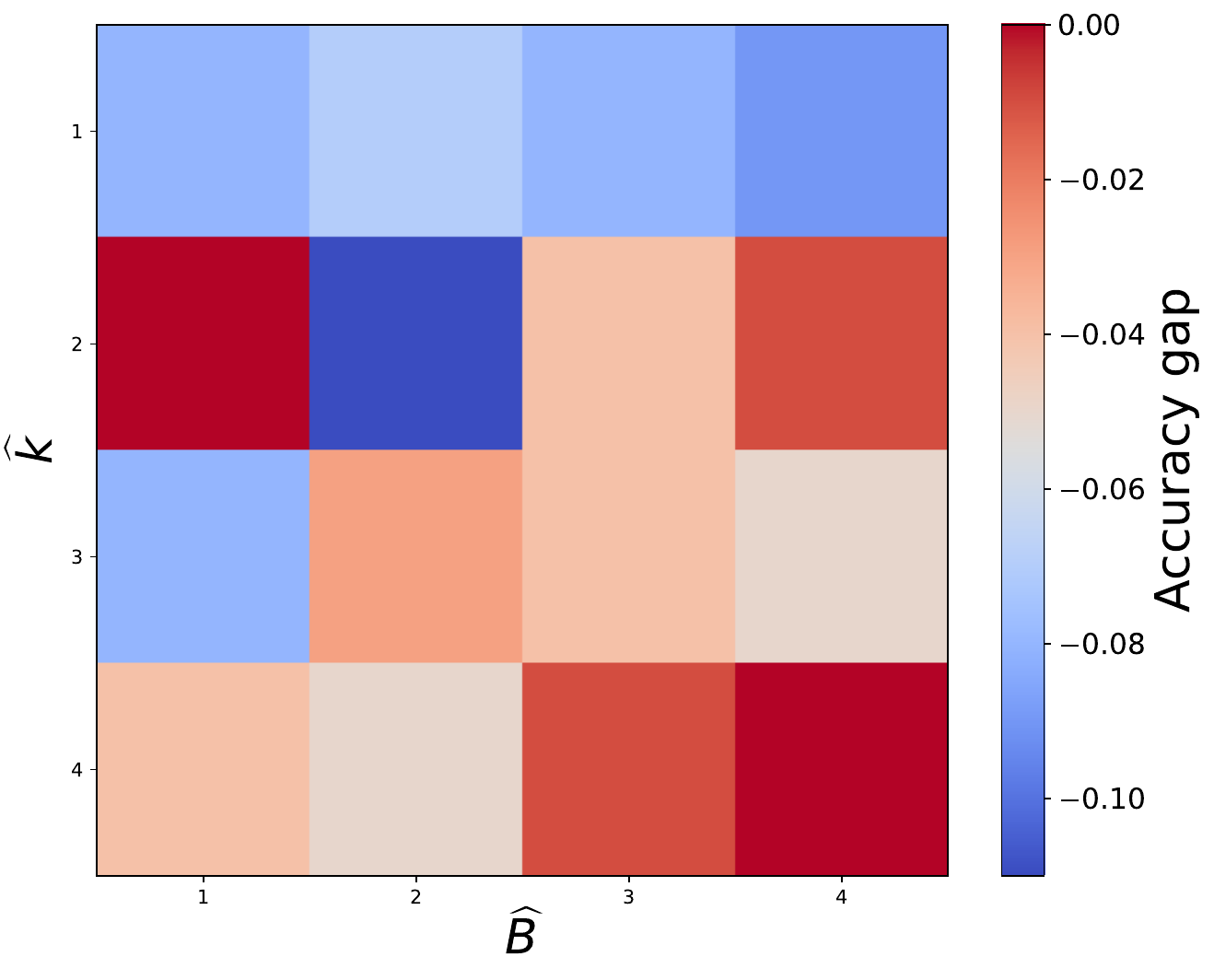}}
    \subfigure[Conf.\ --- Token Gap]{
        \includegraphics[width=0.3\linewidth]{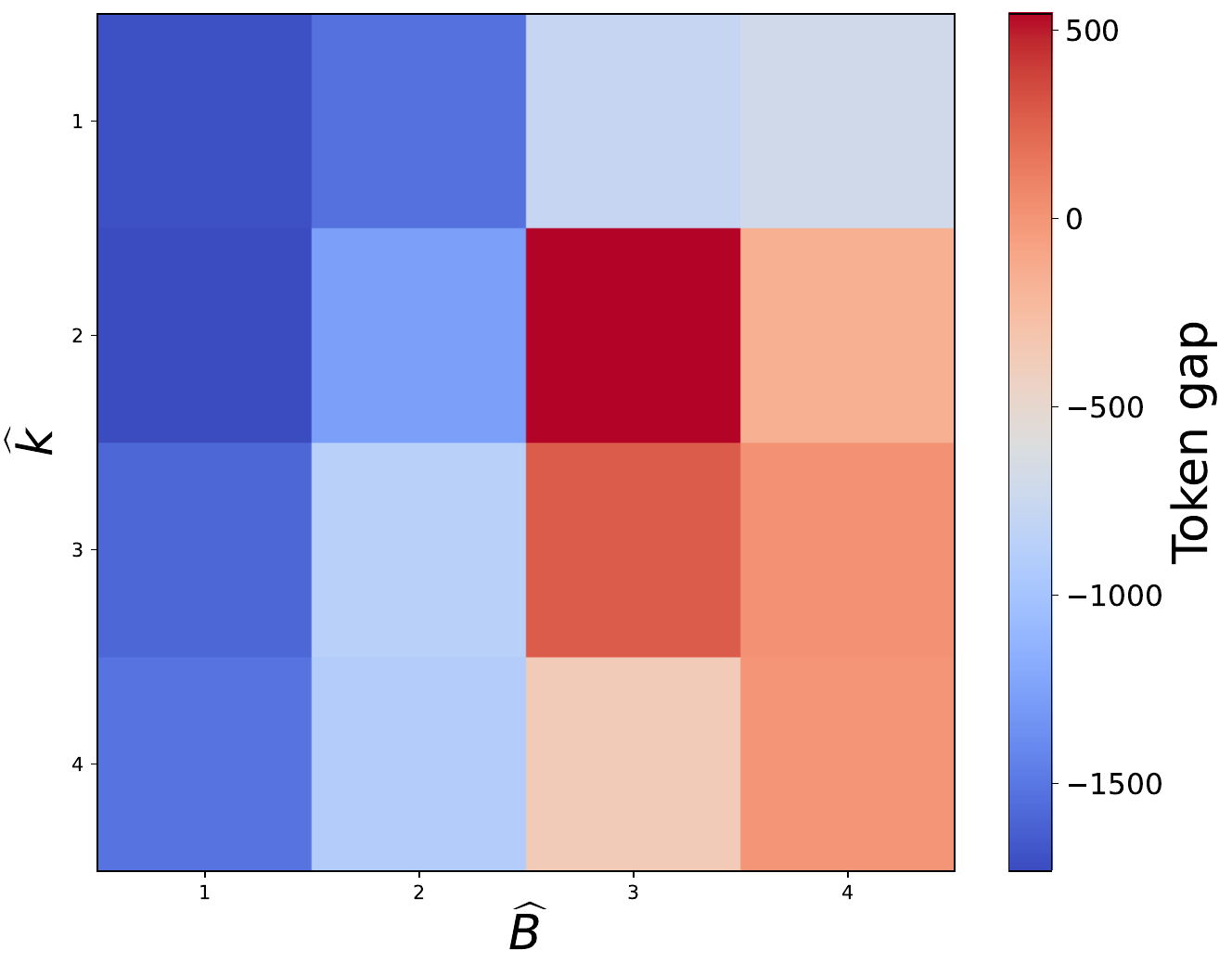}}
    \subfigure[Conf.\ --- Time Gap]{
        \includegraphics[width=0.3\linewidth]{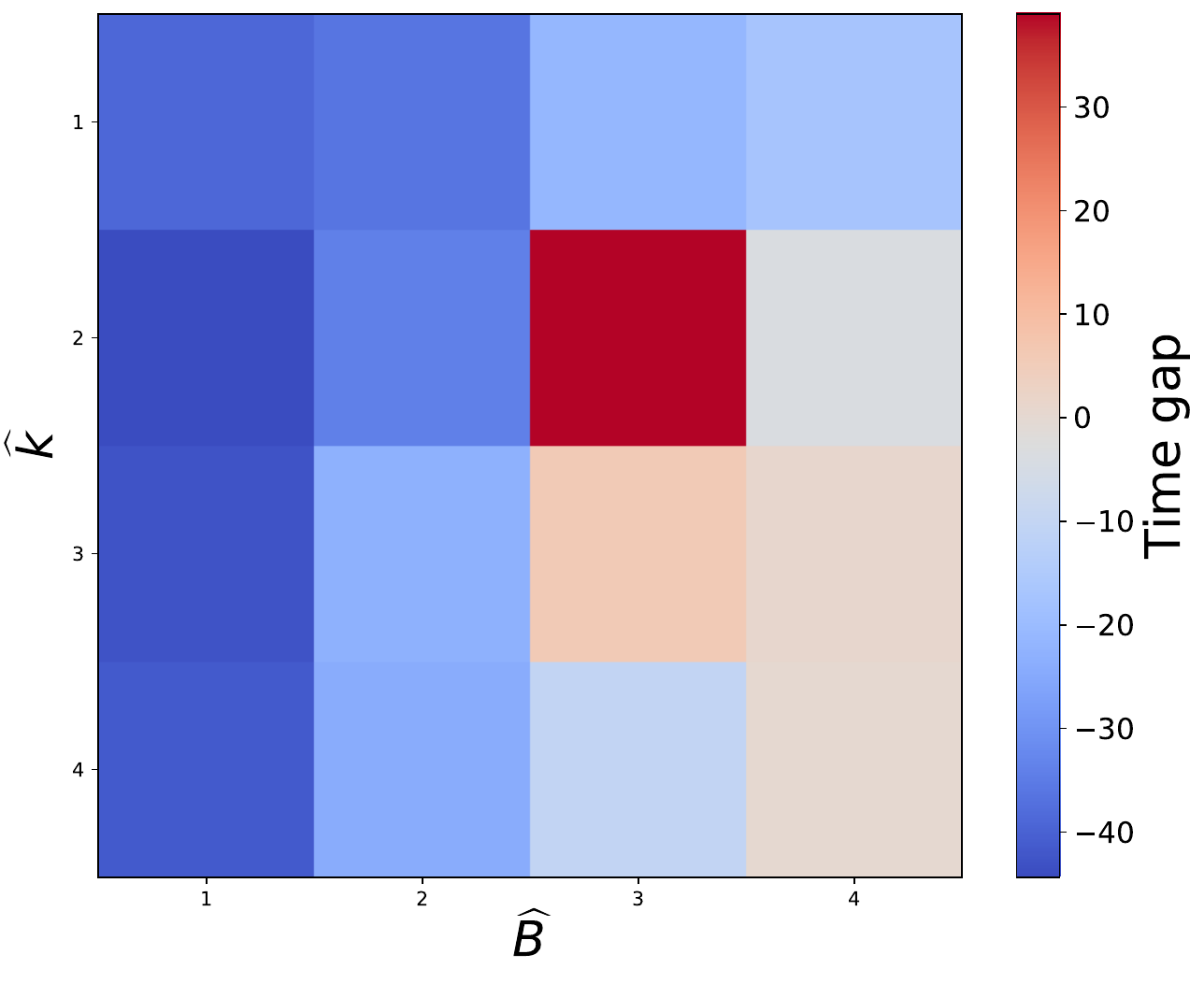}}
    \\
    \subfigure[PRM --- Acc.\ Gap]{
        \includegraphics[width=0.3\linewidth]{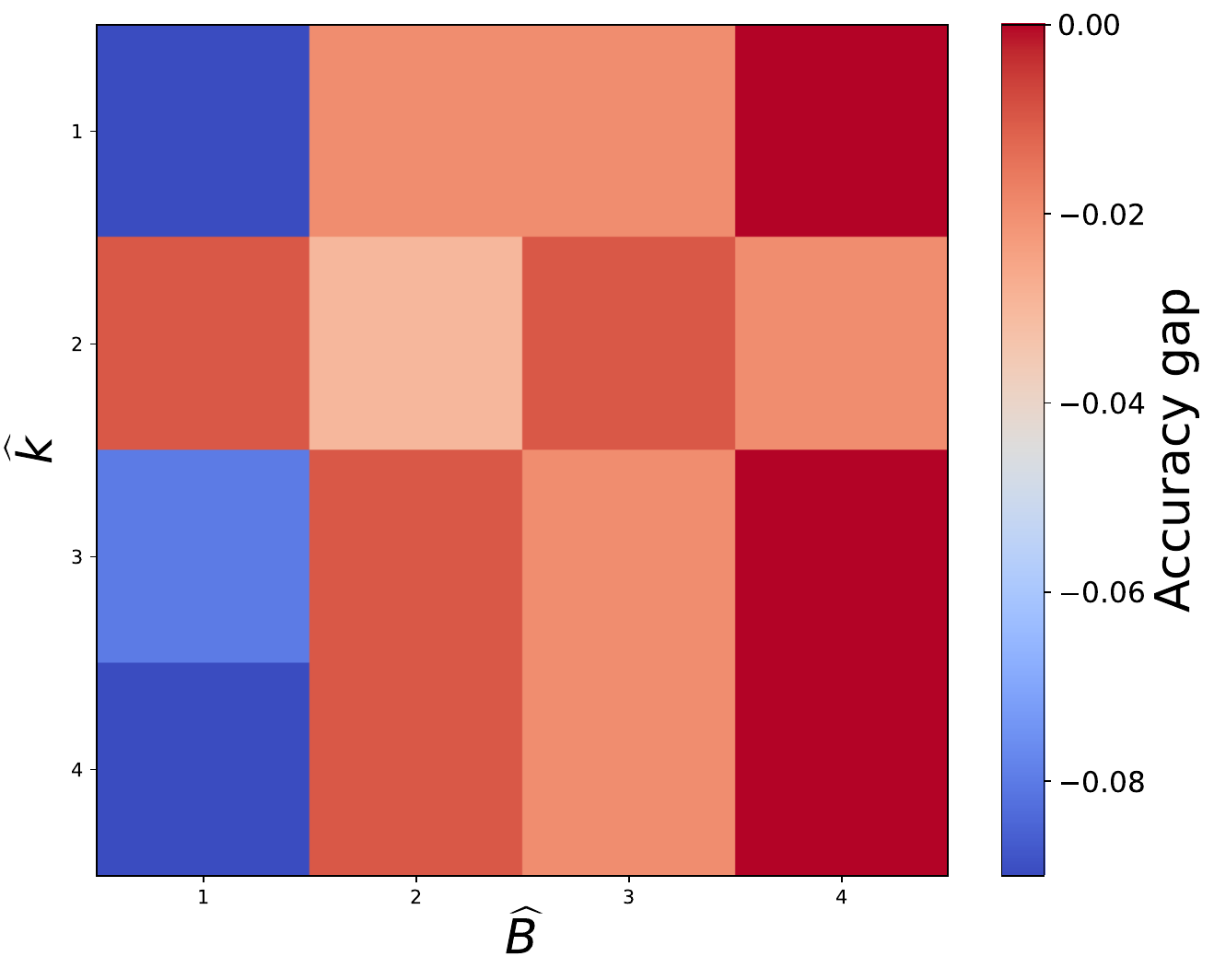}}
    \subfigure[PRM --- Token Gap]{
        \includegraphics[width=0.3\linewidth]{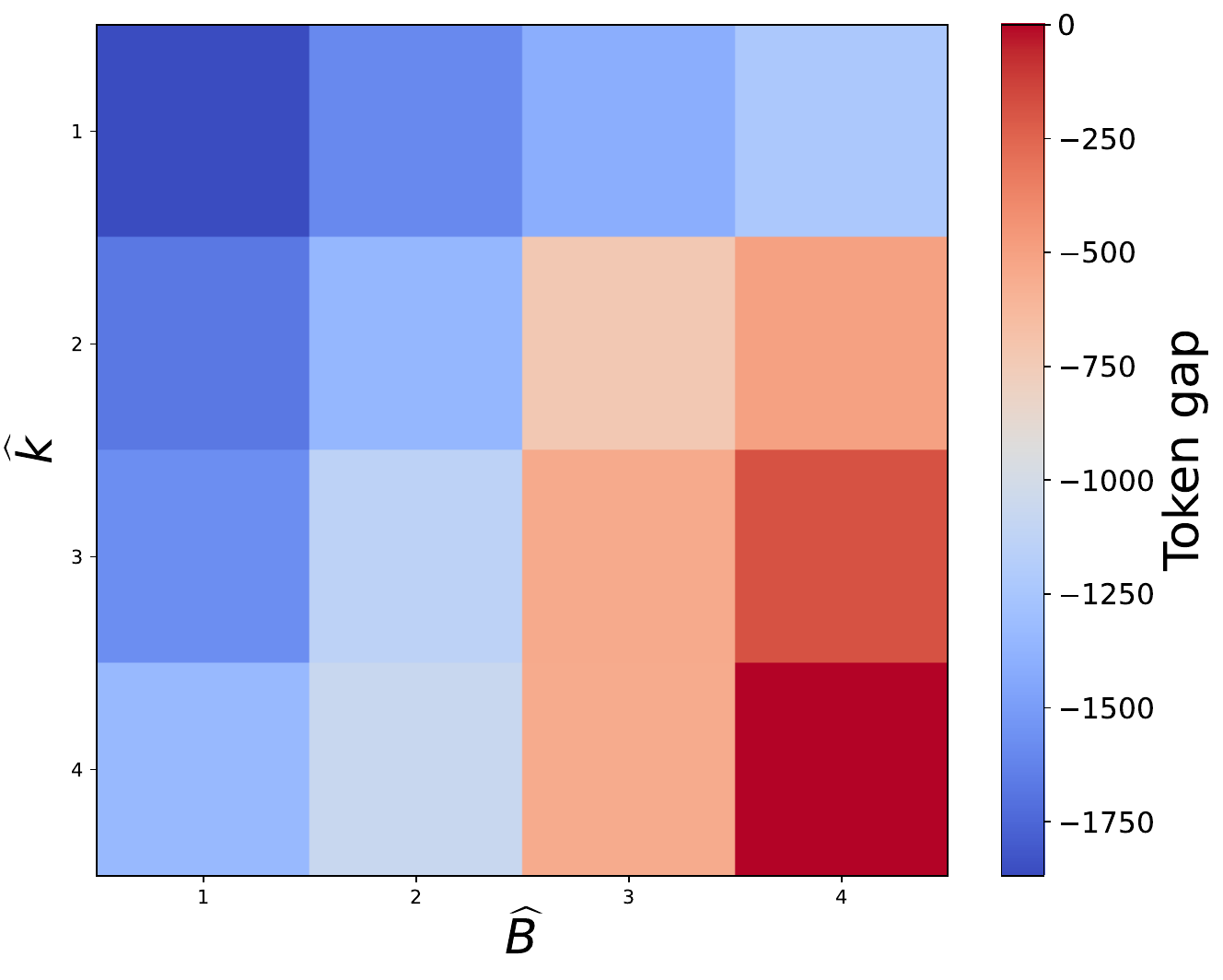}}
    \subfigure[PRM --- Time Gap]{
        \includegraphics[width=0.3\linewidth]{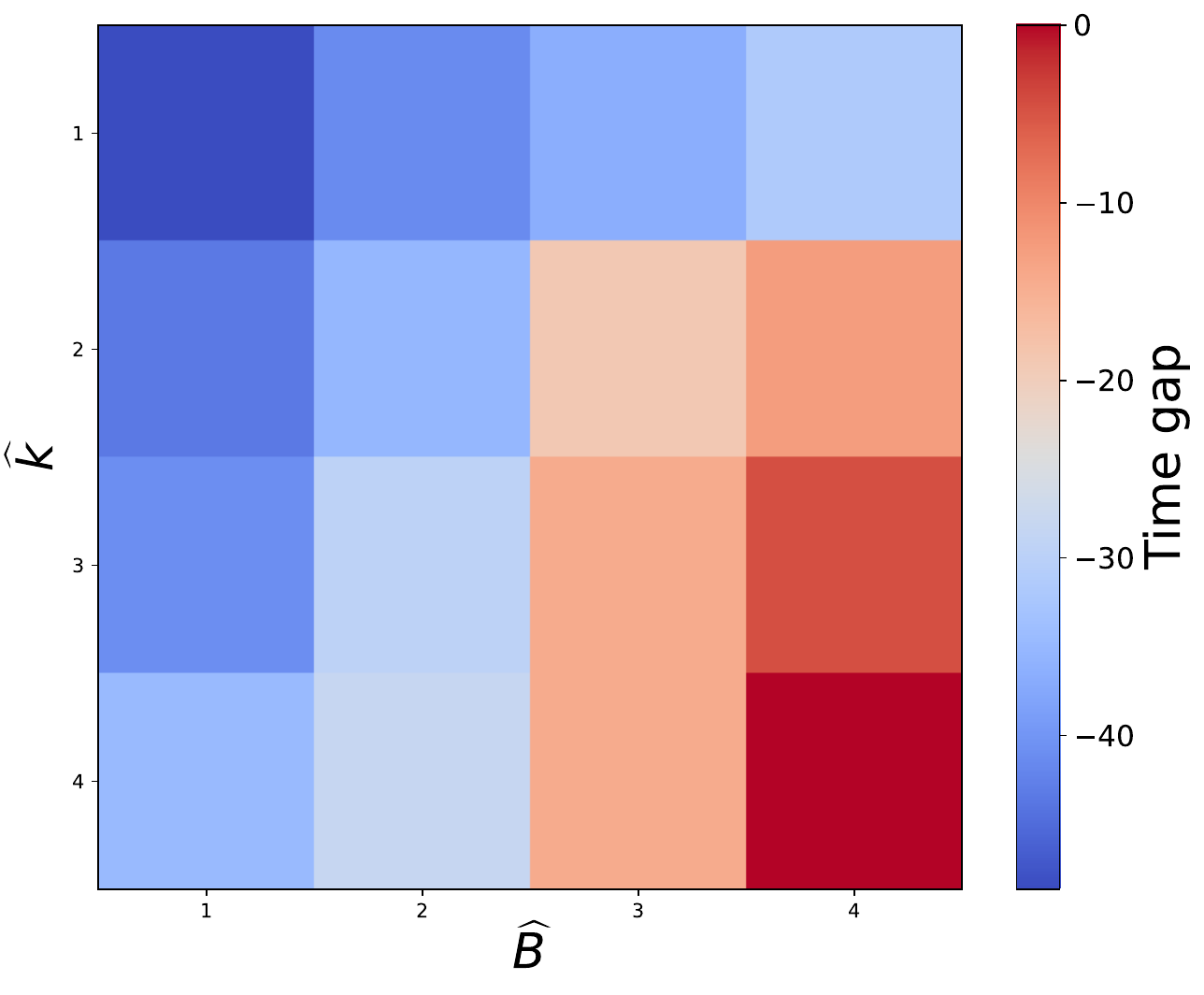}}
    \caption{Hyperparameter sensitivity of $(\widehat{k},\widehat{B})$
    under confidence-guided (top) and PRM-guided (bottom) Beam Search
    on GSM8K. Each cell reports the gap relative to the vanilla Beam
    Search baseline in accuracy, token usage, and wall-clock time.}
    \label{fig:parameters_sensitivity}
\end{figure*}

\section{Ablation study on key hyperparameters}

To understand the influence of the pruning count~$\widehat{k}$ and
the reduced budget~$\widehat{B}$, we conduct a fine-grained ablation
under both Conf.- and PRM-guided Beam Search on GSM8K
(Figure~\ref{fig:parameters_sensitivity}).
Each heatmap reports the gap relative to the vanilla Beam Search
baseline in accuracy, token usage, and wall-clock time.
The results reveal a consistent pattern across both verifiers:
moderate values near the centre of the grid achieve the smallest
accuracy drop or even slight improvements, while increasing~$\widehat{k}$
and decreasing~$\widehat{B}$ yield substantial efficiency savings.
Crucially, \textbf{setting $\widehat{k}=k/2$ and $\widehat{B}=B/2$
consistently hits the sweet spot} that is aggressive enough to cut
cost but not so aggressive as to harm accuracy.

\end{document}